\documentclass[letterpaper]{article} % DO NOT CHANGE THIS
\usepackage{aaai25}  % DO NOT CHANGE THIS
\usepackage{times}  % DO NOT CHANGE THIS
\usepackage{helvet}  % DO NOT CHANGE THIS
\usepackage{courier}  % DO NOT CHANGE THIS
\usepackage[hyphens]{url}  % DO NOT CHANGE THIS
\usepackage{graphicx} % DO NOT CHANGE THIS
\usepackage{natbib}  % DO NOT CHANGE THIS AND DO NOT ADD ANY OPTIONS TO IT
\usepackage{caption} % DO NOT CHANGE THIS AND DO NOT ADD ANY OPTIONS TO IT
\usepackage{algorithm}
\usepackage{algorithmic}
\usepackage{threeparttable}
\usepackage{graphicx}
\usepackage{microtype}
\usepackage{graphicx}
\usepackage{subfigure}
\usepackage[utf8]{inputenc} % allow utf-8 input
\usepackage[T1]{fontenc}    % use 8-bit T1 fonts

\usepackage{url}            % simple URL typesetting
\usepackage{booktabs}       % professional-quality tables
\usepackage{amsfonts}       % blackboard math symbols
\usepackage{nicefrac}       % compact symbols for 1/2, etc.
\usepackage{microtype}      % microtypography
\usepackage{lipsum}
\usepackage{fancyhdr}       % header
\usepackage{graphicx}       % graphics
\graphicspath{{media/}}     % organize your images and other figures under media/ folder

\usepackage{arydshln}  % for dashline in table
\usepackage{mathtools}
\usepackage{amsmath,amssymb,amsfonts}

\usepackage{comment}
\usepackage{enumerate}
\usepackage{colortbl}
\usepackage{booktabs}
\usepackage{array}
\usepackage[english]{babel}
\usepackage{amsthm}
\usepackage{mdframed}
\usepackage{float}
\usepackage{tikz}
\usetikzlibrary{shadows}
\theoremstyle{definition}

\newtheorem{definition}{Definition}

\theoremstyle{remark}
\newtheorem{remark}{Remark}

\newmdtheoremenv{theorem}{Theorem} 
\newmdtheoremenv{prop}{Proposition}
\newmdtheoremenv{lemma}{Lemma}

\usepackage{multirow}
\usepackage{adjustbox}

\usepackage{pifont}

\usepackage{newfloat}
\usepackage{listings}
\DeclareCaptionStyle{ruled}{labelfont=normalfont,labelsep=colon,strut=off} % DO NOT CHANGE THIS
\floatstyle{ruled}
\newfloat{listing}{tb}{lst}{}
\floatname{listing}{Listing}
\title{Sequence Complementor: Complementing Transformers For Time Series Forecasting with Learnable Sequences}
\author{
    Xiwen Chen\textsuperscript{\rm 1}\equalcontrib, Peijie Qiu\textsuperscript{\rm 2}\equalcontrib, Wenhui Zhu\textsuperscript{\rm 3}\equalcontrib, Huayu Li\textsuperscript{\rm 4}, Hao Wang\textsuperscript{\rm 1}, \\ Aristeidis Sotiras\textsuperscript{\rm 2}, Yalin Wang\textsuperscript{\rm 3}, Abolfazl Razi\textsuperscript{\rm 1}
}
\affiliations{
    \textsuperscript{\rm 1} Clemson University, 
\textsuperscript{\rm 2} Washington University in St. Louis, \textsuperscript{\rm 3}Arizona State University, \textsuperscript{\rm 4} The University of Arizona \\
\{xiwenc, hao9\}@g.clemson.edu, \{peijie.qiu, aristeidis.sotiras\}@wustl.edu,\\ \{wzhu59, ylwang\}@asu.edu, hl459@arizona.edu, arazi@clemson.edu
}

\usepackage{bibentry}
\begin{document}

\maketitle

\begin{abstract}
Since its introduction, the transformer has shifted the development trajectory away from traditional models (e.g., RNN, MLP) in time series forecasting, which is attributed to its ability to capture global dependencies within temporal tokens. Follow-up studies have largely involved altering the tokenization and self-attention modules to better adapt Transformers for addressing special challenges like non-stationarity, channel-wise dependency, and variable correlation in time series. However, we found that the expressive capability of sequence representation is a key factor influencing Transformer performance in time forecasting after investigating several representative methods, where there is an almost linear relationship between sequence representation entropy and mean square error, with more diverse representations performing better. In this paper, we propose a novel attention mechanism with Sequence Complementors and prove feasible from an information theory perspective, where these learnable sequences are able to provide complementary information beyond current input to feed attention. We further enhance the Sequence Complementors via a diversification loss that is theoretically covered. The empirical evaluation of both long-term and short-term forecasting has confirmed its superiority over the recent state-of-the-art methods.
\end{abstract}

\section{Introduction}

Time series forecasting (TSF) plays a crucial role in various domains, enabling the extraction of meaningful patterns and the prediction of future events based on historical data. It incubates a wide spectrum of applications, ranging from financial risk assessment, weather forecasting, and energy sustainability to healthcare \cite{wang2024deep}. Recently, the introduction of transformers~\cite{vaswani2017attention} has dramatically shifted the trajectory of model designs for TSF largely due to their ability to capture long-range dependencies. 

% However, we argue that the current design of transformers for TSF is not optimal. 

% In this paper, we will investigate an inescapable, popular, and powerful method in time series analysis — Transformer. 

% Transformers, originally designed for natural language processing (NLP) , has shown remarkable potential in time series analysis. The \textit{ self-attention} mechanism is the main reason behind this success, which is able to model long-range dependencies and parallelize computations of sequential data. These features are particularly well-suited for handling temporal dependency modeling in time series and optimizing the complexity of processing long sequences. 

\begin{figure}
    \centering
    \includegraphics[width=0.98\columnwidth]{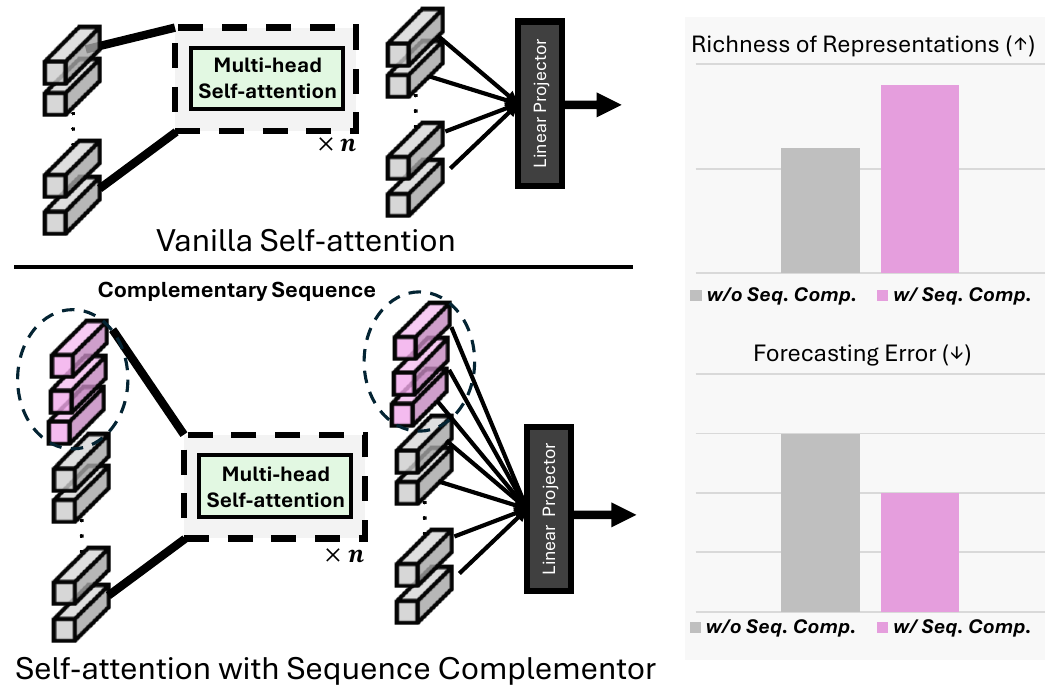}
    \caption{The vanilla self-attention mechanism v.s. the self-attention mechanism with proposed learnable \texttt{Sequence Complementors}, which serve as complementary sequences to the original input sequence (\textbf{Left}).
    The integration of Sequence Complementors results in richer learned representations and a better forecasting performance (\textbf{Right}). }
    \label{fig:intro}
\end{figure}

Following this trend, numerous efforts have been devoted to enhancing the architectural designs of transformer-based TSF models~\cite{zhou2021informer,liu2021pyraformer,zhou2022fedformer,wu2021autoformer,liu2022non,liu2024itransformer}. Although they have achieved satisfactory performance, their effectiveness in TSF is empirically challenged even by simple linear models~\cite{zeng2023transformers,lin2024sparsetsf}. However, an in-depth analysis of what limits the performance of transformers for TSF is insufficient in the literature. This motivates us to investigate the fundamental limits of transformers for TSF in both an empirical and theoretical way. 

% However, their performances are very various, and early transformer-based methods even are questioned by linear methods \cite{zeng2023transformers}\footnote{We believe their claim to be an overstatement and somewhat unfair, and we have discussion about it in Appendix \textcolor{red}{D}. }. Therefore, we are more interested in the phenomena resulting from different implementation of these transformers, rather than focusing solely on their enhancements in forecasting accuracy.

% However, self-attention does not make strong assumptions about data structure, resulting in Transformers typically requiring a vast amount of diverse training samples. We note that the current time series datasets are often small (e.g., around 8000 training pairs in ETTh datasets and around 5000 training pairs in the Exchange dataset), which may suppress the full learning ability of the transformer in time series due to the limited data. 

% Therefore, recent 
% The recent transformer-based models include Informer \cite{zhou2021informer}, Pyraformer \cite{liu2021pyraformer}, Fedformer \cite{zhou2022fedformer}, Autoformer \cite{wu2021autoformer}, Non-stationary Transformer \cite{liu2022non}, PatchTST \cite{nie2023a}, iTransformer \cite{liu2024itransformer}, etc.,.

 % Different models are often developed with various modifications of patchify strategy 

For this purpose, we first conduct a series of empirical observations on four different types of transformers for TSF. We observe that the learned latent representations in four different types of transformers are not rich and generalizable enough, measured by the rank and entropy of the learned feature embeddings. Interestingly, we find that this has a direct relationship with the performance of the TSF task: more diverse and richer latent representations typically lead to higher performance. We believe that this phenomenon is likely due to the combination of limited training samples and domain shifts in most TSF datasets. 
% quantify the richness of the learned representation of transformers via both rank ratio and information-theoretical measures during training. After evaluating different methods, we observe that the representation will reach an information ceiling, and interestingly, after evaluating different transformer-based models, we find a higher information ceiling often translates to a better forecasting performance.

Based on these empirical observations, the most direct solution to enhance the performance of transformers for TSF is to learn rich and generalizable representations. However, this typically requires a larger scale dataset containing diverse samples~\cite{bengio2013representation,lecun2015deep}, which is impractical for our tasks. This inspires us to consider artificially enriching the dataset by injecting a learnable sequence into the raw input sequence. We term this mechanism as~\texttt{Sequence Complementor}, as it complements transformers to discover additional information from the input sequence by interacting with the complementary sequence in a self-attention (see Fig.~\ref{fig:intro}).
Although this solution is simple, it has nice information-theoretic guarantees in lowering the bound of mean-squared error and explicitly increasing the diversity of the learned latent representations. To further enforce the learnable complementary sequence to be diverse, we propose a diversification loss by leveraging the volume maximization of sub-matrices. Our extensive empirical evaluations across a variety of TSF tasks demonstrate the effectiveness of the proposed method.%[see e.g.,~\cite{kulesza2012determinantal,petit2023water}]

% To this end, we propose learnable \texttt{Sequence Compensators}, which serve as complementary input of the original input sequence. Both Sequence Compensators and the original input sequence will participate in the self-attention computation. Subsequently, to preserve the efficiency, we still feed the representation with the original sequence length to the final linear projector. 

% To further enhance the Sequence Compensators, we propose an efficient loss to orthogonalize them through matrix volume maximization. These components establish a novel transformer-based method, named \texttt{DivCompFormer}. The empirical results demonstrate that by simply adding the Sequence Compensators, our method can easily outperform recent state-of-the-art transformer methods as well as methods based on other architectures, such as convolutional neural networks (CNNs) and multi-layer perceptions (MLPs).

In summary, our contributions are threefold: \textbf{(i)} We demonstrate a strong correlation between forecasting performance and the learned representations of different transformer variants; \textbf{(ii)} To this end, we propose a theoretically sound \texttt{Sequence Complementor} to enrich the learned representations of Transformer; \textbf{(iii)} We further propose a novel loss to diversify the learnable Sequence Complementor from a theoretical perspective. 
% \begin{itemize}
%     \item s
% \end{itemize}

% \section{Methodology}
\begin{figure*}[!t]
    \centering
    \includegraphics[width=1.0\textwidth]{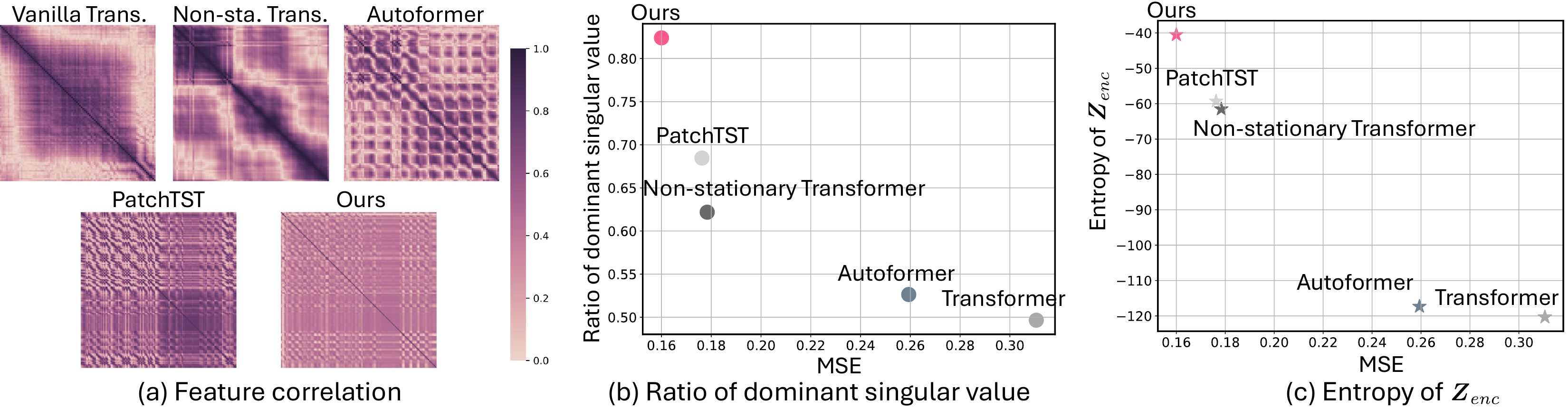}
    % \vspace{-0.2cm}
    \caption{The analysis of transformers for time series forecasting: (a) the correlation of the learned latent representations from the encoder ($\boldsymbol{Z}_{enc}$), (b) the ratio of dominant singular value against MSE, and (c) the feature entropy against MSE.  }
    % \vspace{-0.1cm}
    \label{fig:obs1}
\end{figure*}

\section{Preliminaries}
Without loss of generality, the multivariate time series forecasting (MTSF) given historical observations $\boldsymbol{X}=\left\{\boldsymbol{x}_1, \ldots, \boldsymbol{x}_{T_{in}}\right\} \in \mathbb{R}^{T_{in} \times N}$ with $T$ time steps and $N$ variates, the goal is to predict the future $T_{out}$ time steps $\boldsymbol{Y}=\left\{\boldsymbol{x}_{T_{in}+1}, \ldots, \boldsymbol{x}_{T_{in}+T_{out}}\right\} \in$ $\mathbb{R}^{T_{out} \times N}$.
We consider transformer-based models with a general structure for MTSF~\cite{wu2021autoformer,liu2022non,nie2023a}: 
\begin{align}\label{eq:trans}
    \boldsymbol{X}\xrightarrow{f_\texttt{p}} \boldsymbol{Z}_0 \xrightarrow{f_\texttt{embed}} \boldsymbol{Z}_{embed}  \xrightarrow{f_\texttt{Enc}} \boldsymbol{Z}_{enc}  \xrightarrow{f_\texttt{Dec}} \hat{\boldsymbol{Y}},
\end{align}
where we denote the feature of each sequence $\boldsymbol{X}$ after patchifying as $\boldsymbol{Z}_0 \in \mathbb{R}^{L \times D_0}$, with $L$ and $D_0$ being the number of patchified sequences and the initial dimension of patches, respectively. $\boldsymbol{Z}_{embed}\in \mathbb{R}^{L \times D}$ denotes the embedded patches. Likewise, $\boldsymbol{Z}_{enc}\in \mathbb{R}^{L \times D} $ denotes the latent feature learned by the encoder, and $\hat{\boldsymbol{Y}}$ denotes the final prediction. The encoder often involves several transformer blocks. It is worth mentioning that early works (e.g., \cite{wu2021autoformer})  prefer to utilize transformer blocks as the encoder, while recent works \cite{nie2023a} abandon transformer-based decoder and prefer to directly apply a linear projector. In this sense, we also consider the linear projector as \texttt{encoder} without the loss of generalizability.

% \subsection{What limits the transformers in TSF ?}
\section{The Analysis of Transformers in Time Series Forecasting}
The transformer usually suffers information redundancy, which leads to a large number of neurons highly similar~\cite{redundancy,redundancy1,redundancy2}. Some works on network pruning studies typically address the issue of avoiding computational resource waste by considering the removal of neurons through channel shrinkage~\cite{channel,channel2,channel3}. Here, we conjecture that there is a relationship between this information redundancy and the performance of transformer-based methods in MTSF. To validate our hypothesis, we first select several representative models, including the vanilla transformer, Autoformer \cite{wu2021autoformer}, Non-stationary Transformer \cite{liu2022non}, and PatchTST \cite{nie2023a}. Under fair experimental conditions (Appendix C), we visualized feature similarity by selecting the best-performing variant of each model. Specifically, for each model, we calculated the patch-wise similarity after the self-attention block, generating a feature similarity map for each model (Fig. \ref{fig:obs1}(\textbf{Left})). An intriguing observation is that PatchTST, which exhibits relatively lower feature correlation, consistently achieves better performance (lower MSE scores). This suggests that models with less information redundancy, or lower feature similarity and higher diversity, tend to perform better. However, drawing conclusions based solely on visual inspection of feature correlation is insufficient. Here, %of feature maps along the channel dimension
% In this section, we provide a thorough analysis of the transformers in time series forecasting tasks. We provide an empirical and theoretical justification of what limits the performance of transformers in time series forecasting tasks.
% For this purpose, we consider four representative transformers defined by Eq.~(\ref{eq:trans}) for our analysis, including vanilla transformer, Autoformer \cite{wu2021autoformer}, Non-stationary Transformer \cite{liu2022non}, and PatchTST \cite{nie2023a}. 
% In particular, we are interested in the relation between learned latent feature representations ($\boldsymbol{Z}_{enc}$) and prediction error.
% We evaluate the performance using mean-squared error (MSE). % Here,
we introduce two metrics that are used to further analyze the relation between feature diversity and performance ($\boldsymbol{Z}_{enc}$): (i) the ratio of the dominant singular values in $\boldsymbol{Z}_{enc}$ and (ii) the entropy of features.
Given the singular values of $\boldsymbol{Z}_{enc}$ (i.e., $\{\sigma_1, \sigma_2, \dots,\cdots \}$), the ratio of the dominant singular value is calculated as the ration between the dominant singular values ($\sigma_i>0.1$) and the total number of singular values. In other words, this approximately measures the rank of the matrix $\boldsymbol{Z}_{enc}$. 
\begin{definition}
    Let \( \boldsymbol{Z} \in \mathbb{R}^d \) be a random vector that follows a multivariate Gaussian distribution with covariance matrix \( \boldsymbol{\Sigma} \in \mathbb{R}^{d \times d} \). The Gaussian entropy of \( \boldsymbol{Z} \) is defined as:
\begin{align}\label{eq:def1}
    H_{\mathcal{G}}(\boldsymbol{Z}) = \frac{1}{2} \log \left( (2\pi e)^d \det(\boldsymbol{\Sigma}) \right).
\end{align}
\end{definition}
\begin{remark}
    The entropy can be estimated by a set of samples  \( \boldsymbol{Z}=[ \boldsymbol{z}_1, \boldsymbol{z}_2, \ldots, \boldsymbol{z}_n ]^\top\in  \mathbb{R}^{n\times d} \), where each \( \boldsymbol{z}_i \) is an \( n \)-dimensional vector, the Gaussian entropy is computed as:
\begin{align}
    H_{\mathcal{G}}(\boldsymbol{Z}) = \frac{1}{2} \log \left( (2\pi e)^d \det\left( \frac{1}{n} \boldsymbol{Z}^\top \boldsymbol{Z} +\epsilon\boldsymbol{I} \right) \right),
\end{align}
where the covariance $\boldsymbol{\Sigma}$ in Eq. \ref{eq:def1} is approximated by $\frac{1}{n} \sum_{i=1}^{n} \boldsymbol{z}_i \boldsymbol{z}_i^T = \frac{1}{n}\boldsymbol{Z}^\top \boldsymbol{Z}$. $\epsilon$ and  $\boldsymbol{I}$ denote a very small number and identity matrix, respectively, which are used to avoid non-trivial solutions. This approximation is adopted by \cite{ma2007segmentation,yu2020learning,chen2024learning,rd_dpp}. 
\end{remark}
    % We define the entropy of the singular value of the learned latent vectors ($\boldsymbol{Z}_{\text{enc}}$) as singular value entropy, which measures the richness or the diversity of the learned feature representations:
    % \begin{equation}
    %     S = -\sum_{i} p_i \log p_i,
    % \end{equation}
    % where $p_i$ denotes the normalized frequency of singular values in the \( i \)-th bin in the histogram. 
% It is worth noting that this is a surrogate estimation of entropy, as it is challenging to directly compute entropy for multivariate and discrete $\boldsymbol{Z}_{enc}$.
% As suggested by \cite{kim2023vne,weng2022singular}, the entropy of the singular value of a matrix reflects the diversity or richness of the learned latent representations: a higher singular value entropy corresponds to a richer or more diverse feature representation. 
In this analysis, we employ the Gaussian entropy as a surrogate estimation of the richness or diversity of $\boldsymbol{Z}_{enc}$, and a higher entropy corresponds to a richer or more diverse feature representation. We consider patches to come from one sequence as a set of samples. The metrics are measured when the network is well-trained.
Our analysis result reveals insights that a negative linear relationship between feature entropy and the ratio of the dominant singular value with MSE (refer to Fig. \ref{fig:obs1}(\textbf{Center},\textbf{Right})). Specifically, higher feature entropy and a greater ratio of the dominant singular value generally correspond to better performance, as indicated by lower MSE.  In summary, the performance of transformers in TSF is largely constrained by the lack of richness in the learned latent representations, as they struggle to generalize effectively to the domain shifts that are commonly encountered in TSF.
In contrast, diverse feature representations typically contribute to better generalization~\cite{bengio2013representation,lecun2015deep}, which often leads to better performance in TSF tasks.

\section{Sequence Complementor}
Our empirical analysis suggests that the key to enhancing the transformer's performance in time series forecasting revolves around learning rich feature representations. In this section, we answer this question with the proposed sequence complementor, which involves adding learnable complementary sequences to a raw time-sequence input. Although the proposed method is simple, it is surprisingly effective and has nice theoretical guarantees. 

\subsection{Learnable Complementary Sequence}
% Intuitively, the direct way to .
The most effective way to learn rich feature representations is to train the model on more samples~\cite{bengio2013representation,lecun2015deep,sun2017revisiting}. However, this is not always feasible for TSF in practice, where we typically have a variety of TSF tasks, each of which contains only limited samples. An intuitive idea for this challenge is to improve the diversity of the raw input time sequence. To this end, we propose a learnable complementary sequence that adds an additional learnable sequence to the raw input sequence to enhance its diversity. 

The implementation of the proposed learnable complementary sequence is straightforward. Following the definition of the transformer in Eq.~\ref{eq:trans}, we concatenate $K$ number of learnable sequences $\boldsymbol{S}$ to the feature embeddings $\boldsymbol{Z}_0$:
\begin{align}
    \Tilde{\boldsymbol{Z}}_0\in\mathbb{R}^{ (L_c+K) \times P} \leftarrow \texttt{concat}(\boldsymbol{Z}_0,\boldsymbol{S}),
\end{align}
where $\texttt{concat}(\cdot, \cdot)$ denotes concatenation operation, and $L_c$ and $P$ denotes the number of patches for each channel and patchified sequence length, respectively.
Since $\boldsymbol{S}$ is learnable,
adding $\boldsymbol{S}$ directly to $\boldsymbol{Z}_0$ is more straightforward than adding it to the raw input sequence $\boldsymbol{X}$. Although this process can be model-agnostic and seamlessly integrated into most current TSF transformers, the detailed implementation can vary from model to model (e.g., different patchifying strategies will result in different numbers of tokens). As a proof of concept, we consider the patchifying strategy outlined in~\cite{nie2023a} as an example in this paper. This will result in $L_c= (T_{in}-P)/S+2$ number of patches, where $S$ denotes the stride used for patchifying. 

We then add one more embedding layer to project $\Tilde{\boldsymbol{Z}}_0$ to an embedding feature $\boldsymbol{Z}_{embed}$, which is commonly used in many transformer design \cite{vaswani2017attention}:
\begin{align}
    \boldsymbol{Z}_{embed}\in\mathbb{R}^{ (L_c+K) \times D } \leftarrow f_\texttt{linear}(\Tilde{\boldsymbol{Z}}_0),
\end{align}
where $D$ denotes the embedding dimensions. It is noteworthy that the positional encoding $f_{\texttt{pos}}$ is only applied to patches from  $\boldsymbol{X}$ to preserve the relations of original positions:
\begin{align}
\boldsymbol{Z}_{embed}[:L_c]\leftarrow f_{\texttt{pos}}(\boldsymbol{Z}_{embed}[:L_c])+\boldsymbol{Z}_{embed}[:L_c],\\ \nonumber
  \boldsymbol{Z}_{embed} \leftarrow \texttt{concat}( \boldsymbol{Z}_{embed}[:L_c],\boldsymbol{Z}_{embed}[L_c+1:]).
\end{align}
Then, we can feed $\boldsymbol{Z}_{embed}$ to the transformer blocks of the encoder $f_\texttt{enc}$ to obtain latent representations $\boldsymbol{Z}_{enc}$:
\begin{align}
\boldsymbol{Z}_{enc}\in\mathbb{R}^{ (L_c+K) \times D }\leftarrow f_{\texttt{enc}}(\boldsymbol{Z}_{embed}).
\end{align}

% Suppose after some transformer blocks, we only need a linear projector $f_\texttt{linear}$ to make the final prediction. 
We want to reiterate that we do not intend to modify the original network architecture and also to reduce the computation overhead, we only use the patches from the original sequence for the subsequent operations. Since there is enough interaction among complementary patches and original patches through a self-attention mechanism, there would not be substantial information loss if we abandon complementary patches here. For simplicity, we here only show the self-attention mechanism on the last transformer block from the encoder to reveal this fact:
\begin{align}
&\boldsymbol{Z}_{enc}[:L_c]= \text{softmax}\left(\frac{\boldsymbol{Q} \boldsymbol{K}^T}{\sqrt{d_k}}\right)\boldsymbol{V},\\ \nonumber
&\boldsymbol{Q} =  \boldsymbol{Z}_{embed}[:L_c]\boldsymbol{W}_Q,\\ \nonumber
&\boldsymbol{K} = [\boldsymbol{Z}_{embed}[:L_c], \boldsymbol{Z}_{embed}[L_c+1:]]\boldsymbol{W}_K, \\ \nonumber
&\boldsymbol{V} = [\boldsymbol{Z}_{embed}[:L_c], \boldsymbol{Z}_{embed}[L_c+1:]] \boldsymbol{W}_V.
\end{align}
where $\boldsymbol{W}_Q,\boldsymbol{W}_k,\boldsymbol{W}_V$ are learnable parameters. Then the subsequent operations is presented as,
\begin{align}
\hat{\boldsymbol{Y}}= \texttt{reshape}(f_\texttt{dec}(\boldsymbol{Z}_{enc}[:L_c]).
\end{align}

\subsubsection{Complexity analysis.} Adding complementary sequences will not introduce a huge computation overhead in self-attention. Specifically, this process increases the computational complexity from $\mathcal{O}((L_c^2)$  to $\mathcal{O}((L_c+K)(L_c+K))\approx\mathcal{O}(L_c^2+K^2)$. Additionally, the number of newly introduced parameters is negligible- only $K\times P$ for each channel, where $k$ and $P$ are often small (e.g. $K=3, P=16$). Hence, the overall newly introduced computational and memory complexity is acceptable.

\subsection{Theoretical Justification}

Below, we justify this joint from an information-theoretic perspective.
\subsubsection{Justification 1.} We first note that including the sequence complementors has the potential to result in a richer representation. This is because it can increase the entropy $H(\boldsymbol{Z}_{enc})$:
\begin{align}
    H(\boldsymbol{Z}_{enc})\leq  H(\boldsymbol{Z}_{enc},\boldsymbol{S}).
\end{align}
\begin{proof}  Please refer to Appendix {A}.\end{proof}
As suggested by previous observations, a higher entropy $H(\boldsymbol{Z}_{enc})$ indicates diverse feature representations, which shows a strong relationship to forecasting performance.

\subsubsection{Justification 2.} 
We then note that a richer representation $H(\boldsymbol{Z}_{enc})$ can potentially lead to a lower forecasting error.
Here, we first show the relations between MSE and the conditional uncertainty $H(\boldsymbol{Y} |\boldsymbol{Z}_{enc} )$,

\begin{lemma}\label{lemma:1}
    % Considering a linear projector as the decoder, i.e. $\hat{\boldsymbol{Y}}= \texttt{reshape}(\text{Vec}({\boldsymbol{Z}_{enc}})\boldsymbol{W}_{enc})$ and u
    Under Gaussian assumption, the minimum mean-squared error (MMSE), is bounded by,  
   \begin{align}
       \text{MMSE} \overset{}{\geq} \frac{\exp2 H(\boldsymbol{Y} |\boldsymbol{Z}_{enc} )}{2\pi e}.
       % \overset{(b)}{\geq}\frac{\exp2 H(\boldsymbol{Y} |\boldsymbol{X} )}{2\pi e} .
        \end{align}
        Here,  $H(\cdot|\cdot)$ denotes the conditional entropy. 
\end{lemma} %$\text{Vec}(\cdot)$ denotes the vectorization and
  \begin{proof}  Please refer to Appendix {A}.\end{proof}

\begin{remark}
 Lemma~\ref{lemma:1} is supported by our earlier observations that diverse and rich latent representations ($\boldsymbol{Z}_{{enc}}$) generally result in a lower MSE. This reduction in MSE stems from a decrease in conditional entropy $H(\boldsymbol{Y} | \boldsymbol{Z}_{{enc}})$, which reduces the uncertainty of $\boldsymbol{Y}$ given $\boldsymbol{Z}_{{enc}}$.
\end{remark}

Therefore, by adding the complementary sequence, we may proactively lead to a lower conditional entropy,
\begin{theorem}\label{thm:1}
    The integration of the proposed complementary sequence lowers the bound of MMSE:
    \begin{align}
    H(\boldsymbol{Y} |\boldsymbol{Z}_{{enc}},\boldsymbol{S})\leq  H(\boldsymbol{Y} |\boldsymbol{Z}_{{enc}}).
    \end{align}  
\end{theorem}
  \begin{proof}  Please refer to Appendix {A}.\end{proof}
  
These two justifications converge to our intuitive observations, where the entropy $H(\boldsymbol{Z}_{enc})$ denotes the amount of knowledge we know and conditional entropy $H(\boldsymbol{Y}|\boldsymbol{Z}_{{enc}})$ presents the uncertainty about $\boldsymbol{Y}$ after knowing $\boldsymbol{Z}_{{enc}}$. Adding more information to the system, in this case, through complementary sequences, enriches our knowledge about $\boldsymbol{Y}$, and hence potentially leads to improved performance.
  
% An intuitive understanding of this is shown in Fig. \ref{fig:venn}. The left and middle sub-figures illustrate a richer representation (i.e. higher $H(\boldsymbol{Z}_{enc})$) results in a lower  $H(\boldsymbol{Y}|\boldsymbol{Z}_{enc})$, which the right one shows we can achieve this via adding the sequence complementors $H(\boldsymbol{S})$).

% \begin{figure}[]\label{fig:venn}
% \centering
% \includegraphics[width=0.49\textwidth]{figs/VENN.jpg}
% \caption{}
% \end{figure}

\subsection{Diversified Complementary Sequence}
% Our answer to this question is "YES." To this end, we prefer the Sequence Compensators to be diverse and orthogonal with each other.
We can further diversify or orthogonalize the learnable complementary sequence to enforce them to inject more information. 
Although it is easy to initialize the Sequence Complementors from a set of orthogonal bases, they may not necessarily be orthogonal during network optimization. As a consequence, they are likely to converge to a trivial solution where all complementary sequences are highly similar or identical. To tackle this issue, we propose a diversification loss to diversify the complementary sequences by leveraging the volume maximization of sub-matrices \cite{kulesza2012determinantal,petit2023water}. Different from their problems, which try to find an optimized set of indices that maximizes the volume (i.e., a discrete optimization), we propose a differentiable loss that can be learned through backpropagation with any automatic differentiation framework. 
\begin{definition}
    For given Sequence Complementors $\boldsymbol{S}\in \mathbb{R}^{ K \times P}$, the volume is defined as,
    \begin{align}
    \text{vol}(\boldsymbol{S}) = \prod_{i=1}^K (\sigma_{\boldsymbol{S}} )_i, 
    \end{align}
    where $(\sigma_{\boldsymbol{S}} )_i$ is the $i$-th largest non-zero singular value of $\boldsymbol{S}$.
\end{definition}
% We find giving a constraint to $\boldsymbol{S}$ can theoretically leads the orthogonalization of each sequence. Formally,
\begin{theorem}
    If $\forall \|\boldsymbol{S}_i\|=1, i\in [K]$ holds, maximizing $\text{vol}(\boldsymbol{S})$ results in $\boldsymbol{S}_i \perp \boldsymbol{S}_j,$ for all $  \forall i \neq j,i,j \in [K]$.
\end{theorem}
\begin{proof}
    This can be proven via the arithmetic mean-geometric mean (AM-GM) inequality. Please refer to Appendix {A}.
\end{proof}

\begin{table*}[!t]
\centering
\caption{The results of multivariate long-term time series forecasting. We fix the look-back window length to 96 for all methods, and the forecast horizons are set as $\left\{96, 192, 336, 720 \right\}$. We highlight the best results in \textbf{bold} and the second-best results with \underline{underlining}. \textit{Avg.} denotes the average results from all four prediction lengths. All reported results are averaged over 5 runs.}
\label{tab:LONG}
\resizebox{0.95\textwidth}{!}{%
\begin{tabular}{cl|cc|cc|cc|cc|cc|cc|cc|cc|cc|cc|cc} \toprule
\multicolumn{2}{c}{\textbf{Model}} & \multicolumn{2}{c}{\textbf{Ours}} & \multicolumn{2}{c}{\textbf{\begin{tabular}[c]{@{}c@{}}iTransformer\\      ICLR'24\end{tabular}}} & \multicolumn{2}{c}{\textbf{\begin{tabular}[c]{@{}c@{}}Rlinear\\      Arxiv'23\end{tabular}}} & \multicolumn{2}{c}{\textbf{\begin{tabular}[c]{@{}c@{}}PatchTST\\      ICLR'23\end{tabular}}} & \multicolumn{2}{c}{\textbf{\begin{tabular}[c]{@{}c@{}}Crossformer\\      ICLR'23\end{tabular}}} & \multicolumn{2}{c}{\textbf{\begin{tabular}[c]{@{}c@{}}TimesNet\\      ICLR'23\end{tabular}}} & \multicolumn{2}{c}{\textbf{\begin{tabular}[c]{@{}c@{}}Dlinear\\      AAAI'23\end{tabular}}} & \multicolumn{2}{c}{\textbf{\begin{tabular}[c]{@{}c@{}}SCINet\\      NeurIPS'22\end{tabular}}} & \multicolumn{2}{c}{\textbf{\begin{tabular}[c]{@{}c@{}}FEDformer\\      ICML'22\end{tabular}}} & \multicolumn{2}{c}{\textbf{\begin{tabular}[c]{@{}c@{}}Stationary\\      NeurIPS'22\end{tabular}}} & \multicolumn{2}{c}{\textbf{\begin{tabular}[c]{@{}c@{}}Autoformer\\      NeurIPS'21\end{tabular}}} \\ \midrule
\multicolumn{2}{c}{Metric} & MSE & MAE & MSE & MAE & MSE & MAE & MSE & MAE & MSE & MAE & MSE & MAE & MSE & MAE & MSE & MAE & MSE & MAE & MSE & MAE & MSE & MAE \\ \midrule
 & 96 & \textbf{ 0.321} & \textbf{ 0.359} & 0.334 & 0.368 & 0.355 & 0.376 & \underline{ 0.329} & \underline{ 0.367} & 0.404 & 0.426 & 0.338 & 0.375 & 0.345 & 0.372 & 0.418 & 0.438 & 0.379 & 0.419 & 0.386 & 0.398 & 0.505 & 0.475 \\
 & 192 & \textbf{ 0.362} & \underline{ 0.386} & 0.377 & 0.391 & 0.391 & 0.392 & \underline{ 0.367} & \textbf{ 0.385} & 0.450 & 0.451 & 0.374 & 0.387 & 0.380 & 0.389 & 0.439 & 0.450 & 0.426 & 0.441 & 0.459 & 0.444 & 0.553 & 0.496 \\
 & 336 & \textbf{ 0.393} & \textbf{ 0.406} & 0.426 & 0.420 & 0.424 & 0.415 & \underline{ 0.399} & \underline{ 0.410} & 0.532 & 0.515 & 0.410 & 0.411 & 0.413 & 0.413 & 0.490 & 0.485 & 0.445 & 0.459 & 0.495 & 0.464 & 0.621 & 0.537 \\
 & 720 & \textbf{ 0.450} & \underline{ 0.442} & 0.491 & 0.459 & 0.487 & 0.450 & \underline{ 0.454} & \textbf{ 0.439} & 0.666 & 0.589 & 0.478 & 0.450 & 0.474 & 0.453 & 0.595 & 0.550 & 0.543 & 0.490 & 0.585 & 0.516 & 0.671 & 0.561 \\
\multirow{-5}{*}{\rotatebox{90}{ETTm1}} & Avg. & \textbf{ 0.381} & \textbf{ 0.398} & 0.407 & 0.410 & 0.414 & 0.408 & \underline{ 0.387} & \underline{ 0.400} & 0.513 & 0.496& 0.400 & 0.406 & 0.403 & 0.407 & 0.485& 0.481 & 0.448 & 0.452 & 0.481 & 0.456 & 0.588 & 0.517 \\ \midrule
 & 96 & \textbf{ 0.172} & \textbf{ 0.256} & 0.180 & 0.264 & 0.182 & 0.265 & \underline{ 0.175} & \underline{ 0.259} & 0.287 & 0.366 & 0.187 & 0.267 & 0.193 & 0.292 & 0.286 & 0.377 & 0.203 & 0.287 & 0.192 & 0.274 & 0.255 & 0.339 \\
 & 192 & \textbf{ 0.236} & \textbf{ 0.299} & 0.250 & 0.309 & 0.246 & 0.304 & \underline{ 0.241} & \underline{ 0.302} & 0.414 & 0.492 & 0.249 & 0.309 & 0.284 & 0.362 & 0.399 & 0.445 & 0.269 & 0.328 & 0.280 & 0.339 & 0.281 & 0.340 \\
 & 336 & \textbf{ 0.298} & \textbf{ 0.339} & 0.311 & 0.348 & 0.307 & \underline{ 0.342} & \underline{ 0.305} & 0.343 & 0.597 & 0.542 & 0.321 & 0.351 & 0.369 & 0.427 & 0.637 & 0.591 & 0.325 & 0.366 & 0.334 & 0.361 & 0.339 & 0.372 \\
 & 720 & \textbf{ 0.395} & \textbf{ 0.397} & 0.412 & 0.407 & 0.407 & \underline{ 0.398} & \underline{ 0.402} & 0.400 & 1.730 & 1.042 & 0.408 & 0.403 & 0.554 & 0.522 & 0.960 & 0.735 & 0.421 & 0.415 & 0.417 & 0.413 & 0.433 & 0.432 \\
\multirow{-5}{*}{\rotatebox{90}{ETTm2}} & Avg. & \textbf{ 0.275} & \textbf{ 0.323} & 0.288 & 0.332 & 0.286 & 0.327 & \underline{ 0.281} & \underline{ 0.326} & 0.757 & 0.610& 0.291 & 0.333 & 0.350 & 0.401 & 0.571 & 0.537 & 0.305 & 0.349 & 0.306 & 0.347 & 0.327 & 0.371 \\ \midrule
 & 96 & \textbf{ 0.375} & \textbf{ 0.394} & 0.386 & 0.405 & 0.386 & \underline{ 0.395} & 0.414 & 0.419 & 0.423 & 0.448 & 0.384 & 0.402 & 0.386 & 0.400 & 0.654 & 0.599 & \underline{ 0.376} & 0.419& 0.513 & 0.491 & 0.449 & 0.459 \\
 & 192 & \underline{ 0.423} & \underline{ 0.425} & 0.441 & 0.436 & 0.437 & \textbf{ 0.424} & 0.460 & 0.445 & 0.471 & 0.474 & 0.436 & 0.429 & 0.437 & 0.432 & 0.719 & 0.631 & \textbf{ 0.420} & 0.448 & 0.534 & 0.504 & 0.500 & 0.482 \\
 & 336 & \textbf{ 0.457} & \underline{ 0.448} & 0.487 & 0.458 & 0.479 & \textbf{ 0.446} & 0.501 & 0.466 & 0.570 & 0.546 & 0.491 & 0.469 & 0.481 & 0.459 & 0.778 & 0.659 & \underline{ 0.459} & 0.465 & 0.588 & 0.535 & 0.521 & 0.496 \\
 & 720 & \textbf{ 0.462} & \underline{ 0.472} & 0.503 & 0.491 & \underline{ 0.481} & \textbf{ 0.470} & 0.500 & 0.488 & 0.653 & 0.621 & 0.521 & 0.500 & 0.519 & 0.516 & 0.836 & 0.699 & 0.506 & 0.507 & 0.643 & 0.616 & 0.514 & 0.512 \\
\multirow{-5}{*}{\rotatebox{90}{ETTh1}} & Avg. & \textbf{ 0.429} & \underline{ 0.435} & 0.454 & 0.448 & 0.446 & \textbf{ 0.434} & 0.469 & 0.454& 0.529 & 0.522 & 0.458 & 0.450 & 0.456 & 0.452 & 0.747 & 0.647 & \underline{ 0.440} & 0.460& 0.570 & 0.537 & 0.496 & 0.487 \\ \midrule
 & 96 & \textbf{ 0.283} & \textbf{ 0.337} & 0.297 & 0.349 & \underline{ 0.288} & \underline{ 0.338} & 0.302 & 0.348 & 0.745 & 0.584 & 0.340 & 0.374 & 0.333 & 0.387 & 0.707 & 0.621 & 0.358 & 0.397 & 0.476 & 0.458 & 0.346 & 0.388 \\
 & 192 & \textbf{ 0.362} & \textbf{ 0.388} & 0.380 & 0.400 & \underline{ 0.374} & \underline{ 0.390} & 0.388 & 0.400 & 0.877 & 0.656 & 0.402 & 0.414 & 0.477 & 0.476 & 0.860 & 0.689 & 0.429 & 0.439 & 0.512 & 0.493 & 0.456 & 0.452 \\
 & 336 & \textbf{ 0.408} & \textbf{ 0.425} & 0.428 & 0.432 & \underline{ 0.415} & \underline{ 0.426} & 0.426 & 0.433 & 1.043 & 0.731 & 0.452 & 0.452 & 0.594 & 0.541 & 1.000 & 0.744 & 0.496 & 0.487 & 0.552 & 0.551 & 0.482 & 0.486 \\
 & 720 & \textbf{ 0.419} & \underline{ 0.441} & 0.427 & 0.445 & \underline{ 0.420} & \textbf{ 0.440} & 0.431 & 0.446 & 1.104 & 0.763 & 0.462 & 0.468 & 0.831 & 0.657 & 1.249 & 0.838 & 0.463 & 0.474 & 0.562 & 0.560 & 0.515 & 0.511 \\
\multirow{-5}{*}{\rotatebox{90}{ETTh2}} & Avg. & \textbf{ 0.368} & \textbf{ 0.398} & 0.383 & 0.407 & \underline{ 0.374} & \underline{ 0.399} & 0.387 & 0.407 & 0.942 & 0.684 & 0.414 & 0.427 & 0.559 & 0.515 & 0.954 & 0.723 & 0.437 & 0.449 & 0.526 & 0.516 & 0.450 & 0.459 \\ \midrule
 & 96 & \underline{ 0.156} & \underline{ 0.252} & \textbf{ 0.148} & \textbf{ 0.240} & 0.201 & 0.281 & 0.181 & 0.270 & 0.219 & 0.314 & 0.168 & 0.272 & 0.197 & 0.282 & 0.247 & 0.345 & 0.193 & 0.308 & 0.169 & 0.273 & 0.201 & 0.317 \\
 & 192 & \underline{ 0.175} & \underline{ 0.269} & \textbf{ 0.162} & \textbf{ 0.253} & 0.201 & 0.283 & 0.188 & 0.274 & 0.231 & 0.322 & 0.184 & 0.289 & 0.196 & 0.285 & 0.257 & 0.355 & 0.201 & 0.315 & 0.182 & 0.286 & 0.222 & 0.334 \\
 & 336 & \underline{ 0.190} & \underline{ 0.284} & \textbf{ 0.178} & \textbf{ 0.269} & 0.215 & 0.298 & 0.204 & 0.293 & 0.246 & 0.337 & 0.198 & 0.300 & 0.209 & 0.301 & 0.269 & 0.369 & 0.214 & 0.329 & 0.200 & 0.304 & 0.231 & 0.338 \\
 & 720 & 0.246 & 0.324 & 0.225 & \textbf{ 0.317} & 0.257 & 0.331 & 0.246 & 0.324 & 0.280 & 0.363 & \textbf{ 0.220} & \underline{ 0.320} & 0.245 & 0.333 & 0.299 & 0.390 & 0.246 & 0.355 & \underline{ 0.222} & 0.321 & 0.254 & 0.361 \\
\multirow{-5}{*}{\rotatebox{90}{ECL}} & Avg. & \underline{ 0.192} & \underline{ 0.282} & \textbf{ 0.178} & \textbf{ 0.270} & 0.219 & 0.298 & 0.205 & 0.290 & 0.244 & 0.334 & 0.192& 0.295 & 0.212 & 0.300 & 0.268 & 0.365 & 0.214 & 0.327 & 0.193 & 0.296 & 0.227 & 0.338 \\ \midrule
 & 96 & \textbf{ 0.082} & \textbf{ 0.200} & \underline{ 0.086} & 0.206 & 0.093 & 0.217 & 0.088 & \underline{ 0.205} & 0.256 & 0.367 & 0.107 & 0.234 & 0.088 & 0.218 & 0.267 & 0.396 & 0.148 & 0.278 & 0.111 & 0.237 & 0.197 & 0.323 \\
 & 192 & \textbf{ 0.175} & \textbf{ 0.297} & 0.177 & \underline{ 0.299} & 0.184 & 0.307 & \underline{ 0.176} & \underline{ 0.299} & 0.470 & 0.509 & 0.226 & 0.344 & \underline{ 0.176} & 0.315 & 0.351 & 0.459 & 0.271 & 0.315 & 0.219 & 0.335 & 0.300 & 0.369 \\
 & 336 & 0.325 & \underline{ 0.413} & 0.331 & 0.419 & 0.351 & 0.432 & \textbf{ 0.301} & \textbf{ 0.397} & 1.268 & 0.883 & 0.367 & 0.448 & \underline{ 0.313} & 0.427 & 1.324 & 0.853 & 0.460 & 0.427 & 0.421 & 0.476 & 0.509 & 0.524 \\
 & 720 & \underline{ 0.840} & \textbf{ 0.690} & 0.847 & \underline{ 0.691} & 0.886 & 0.714 & 0.901 & 0.714 & 1.767 & 1.068 & 0.964 & 0.746 & \textbf{ 0.839} & 0.695 & 1.058 & 0.797 & 1.195 & 0.695 & 1.092 & 0.769 & 1.447 & 0.941 \\
\multirow{-5}{*}{\rotatebox{90}{Exchange}} & Avg. & \underline{ 0.356} & \textbf{ 0.400} & 0.360 & \underline{ 0.404} & 0.378& 0.417& 0.367 & 0.404 & 0.940 & 0.707 & 0.416 & 0.443 & \textbf{ 0.354} & 0.414 & 0.750 & 0.626 & 0.519 & 0.429 & 0.461 & 0.454 & 0.613 & 0.539 \\ \midrule
 & 96 & \underline{ 0.455} & 0.291 & \textbf{ 0.395} & \textbf{ 0.268} & 0.649 & 0.389 & 0.462 & 0.295 & 0.522 & \underline{ 0.290} & 0.593 & 0.321 & 0.650 & 0.396 & 0.788 & 0.499 & 0.587 & 0.366 & 0.612 & 0.338 & 0.613 & 0.388 \\
 & 192 & \underline{ 0.460} & \underline{ 0.291} & \textbf{ 0.417} & \textbf{ 0.276} & 0.601 & 0.366 & 0.466 & 0.296 & 0.530 & 0.293 & 0.617 & 0.336 & 0.598 & 0.370 & 0.789 & 0.505 & 0.604 & 0.373 & 0.613 & 0.340 & 0.616 & 0.382 \\
 & 336 & \underline{ 0.473} & \underline{ 0.302} & \textbf{ 0.433} & \textbf{ 0.283} & 0.609 & 0.369 & 0.482 & 0.304 & 0.558 & 0.305 & 0.629 & 0.336 & 0.605 & 0.373 & 0.797 & 0.508 & 0.621 & 0.383 & 0.618 & 0.328 & 0.622 & 0.337 \\
 & 720 & \underline{ 0.500} & \underline{ 0.321} & \textbf{ 0.467} & \textbf{ 0.302} & 0.647 & 0.387 & 0.514 & 0.322 & 0.589 & 0.328 & 0.640 & 0.350 & 0.645 & 0.394 & 0.841 & 0.523 & 0.626 & 0.382 & 0.653 & 0.355 & 0.660 & 0.408 \\
\multirow{-5}{*}{\rotatebox{90}{Traffic}} & Avg. & \underline{ 0.472} & \underline{ 0.301} & \textbf{ 0.428} & \textbf{ 0.282} & 0.627 & 0.378 & 0.481 & 0.304 & 0.550 & 0.304 & 0.620 & 0.336 & 0.625 & 0.383 & 0.804 & 0.509 & 0.610 & 0.376 & 0.624 & 0.340 & 0.628 & 0.379 \\ \midrule
 & 96 & \underline{ 0.159} & \textbf{ 0.206} & 0.174 & \underline{ 0.214} & 0.192 & 0.232 & 0.177 & 0.218 & \textbf{ 0.158} & 0.230 & 0.172 & 0.220 & 0.196 & 0.255 & 0.221 & 0.306 & 0.217 & 0.296 & 0.173 & 0.223 & 0.266 & 0.336 \\
 & 192 & \textbf{ 0.205} & \textbf{ 0.249} & 0.221 & \underline{ 0.254} & 0.240 & 0.271 & 0.225 & 0.259 & \underline{ 0.206} & 0.277 & 0.219 & 0.261 & 0.237 & 0.296 & 0.261 & 0.340 & 0.276 & 0.336 & 0.245 & 0.285 & 0.307 & 0.367 \\
 & 336 & \textbf{ 0.263} & \textbf{ 0.291} & 0.278 & \underline{ 0.296} & 0.292 & 0.307 & 0.278 & 0.297 & \underline{ 0.272} & 0.335 & 0.280 & 0.306 & 0.283 & 0.335 & 0.309 & 0.378 & 0.339 & 0.380 & 0.321 & 0.338 & 0.359 & 0.395 \\
 & 720 & \textbf{ 0.344} & \textbf{ 0.345} & 0.358 & \underline{ 0.347} & 0.364 & 0.353 & 0.354 & 0.348 & 0.398 & 0.418 & 0.365 & 0.359 & \underline{ 0.345} & 0.381 & 0.377 & 0.427 & 0.403 & 0.428 & 0.414 & 0.410 & 0.419 & 0.428 \\
\multirow{-5}{*}{\rotatebox{90}{Weather}} & Avg. & \textbf{ 0.243} & \textbf{ 0.273} & \underline{ 0.258} & \underline{ 0.278} & 0.272 & 0.291 & 0.259 & 0.281 & 0.259 & 0.315 & 0.259 & 0.287 & 0.265 & 0.317 & 0.292 & 0.363 & 0.309 & 0.360 & 0.288 & 0.314 & 0.338 & 0.382 \\ \midrule
\multicolumn{2}{c}{First Count} & \multicolumn{2}{c}{\textbf{47}} & \multicolumn{2}{c}{\underline{19}} & \multicolumn{2}{c}{5} & \multicolumn{2}{c}{4} & \multicolumn{2}{c}{1} & \multicolumn{2}{c}{1} & \multicolumn{2}{c}{2} & \multicolumn{2}{c}{0} & \multicolumn{2}{c}{1} & \multicolumn{1}{c}{0} & \multicolumn{1}{c}{} & \multicolumn{2}{c}{0} \\ \bottomrule
\end{tabular}%
}
\end{table*}

% \usepackage{multirow}
% \usepackage{graphicx}
\begin{table*}[!t]
\centering
\caption{The results of multivariate short-term time series forecasting. We highlight the best results in \textbf{bold} and the second-best results with \underline{underlining}. All reported results are averaged over 5 runs. For the details of the metrics, please refer to Appendix {D}.}
\label{tab:SHORT}
\resizebox{0.9\textwidth}{!}{%
\begin{tabular}{clccccccccccc} \toprule
\multicolumn{2}{c}{\textbf{Models}} & \textbf{Ours} & \textbf{\begin{tabular}[c]{@{}c@{}}Rlinear\\      Arxiv'23\end{tabular}} & \textbf{\begin{tabular}[c]{@{}c@{}}PatchTST\\      ICLR'23\end{tabular}} & \textbf{\begin{tabular}[c]{@{}c@{}}Crossformer \\      ICLR'23\end{tabular}} & \textbf{\begin{tabular}[c]{@{}c@{}}FEDformer \\      ICML'22\end{tabular}} & \textbf{\begin{tabular}[c]{@{}c@{}}TimesNet\\      ICLR'23\end{tabular}} & \textbf{\begin{tabular}[c]{@{}c@{}}DLinear \\      AAAI'23\end{tabular}} & \textbf{\begin{tabular}[c]{@{}c@{}}SCINet \\      NeurIPS'22\end{tabular}} & \textbf{\begin{tabular}[c]{@{}c@{}}FEDformer\\      ICML'22\end{tabular}} & \textbf{\begin{tabular}[c]{@{}c@{}}Stationary\\      NeurIPS'22\end{tabular}} & \textbf{\begin{tabular}[c]{@{}c@{}}Autoformer\\      NeurIPS'21\end{tabular}} \\ \midrule
 & SMAPE & \textbf{ 13.185} & 13.994 & \underline{ 13.258} & 13.392 & 13.728 & 13.387 & 16.965 & 13.717 & 13.728 & 13.717 & 13.974 \\
 & MASE & \textbf{ 2.955} & 3.015 & \underline{ 2.985} & 3.001 & 3.048 & 2.996 & 4.283 & 3.076 & 3.048 & 3.078 & 3.134 \\
\multirow{-3}{*}{Year} & OWA & \textbf{ 0.775} & 0.807 & \underline{ 0.781} & 0.787 & 0.803 & 0.786 & 1.058 & 0.807 & 0.803 & 0.807 & 0.822 \\ \midrule
 & SMAPE & \textbf{ 9.989} & 10.702 & 10.179 & 16.317 & 10.792 & \underline{ 10.1} & 12.145 & 10.845 & 10.792 & 10.958 & 11.338 \\
 & MASE &  \textbf{1.17} & 1.299 &  1.212 & 2.197 & 1.283 & \underline{1.182} & 1.52 & 1.295 & 1.283 & 1.325 & 1.365 \\
\multirow{-3}{*}{Quarterly} & OWA &  \textbf{0.88} & 0.959 &  0.904  & 1.542 & 0.958 & \underline{0.89} & 1.106 & 0.965 & 0.958 & 0.981 & 1.012 \\ \midrule
 & SMAPE & \textbf{ 12.453} & 13.363 & \underline{ 12.641} & 12.924 & 14.26 & 12.67 & 13.514 & 13.208 & 14.26 & 13.917 & 13.958 \\
 & MASE & \textbf{ 0.913} & 1.014 & \underline{ 0.93} & 0.966 & 1.102 & 0.933 & 1.037 & 0.999 & 1.102 & 1.097 & 1.103 \\
\multirow{-3}{*}{Monthly} & OWA & \textbf{ 0.861} & 0.94 & \underline{ 0.876} & 0.902 & 1.012 & 0.878 & 0.956 & 0.928 & 1.012 & 0.998 & 1.002 \\ \midrule
 & SMAPE & \textbf{ 4.565} & 5.437 & 4.946 & 5.439 & 4.954 & \underline{ 4.891} & 6.709 & 5.432 & 4.954 & 6.302 & 5.458 \\
 & MASE & \underline{ 3.144} & 3.706 & \textbf{ 2.985} & 3.69 & 3.264 & 3.302 & 4.953 & 3.583 & 3.264 & 4.064 & 3.865 \\
\multirow{-3}{*}{Others} & OWA & \textbf{ 0.976} & 1.157 & 1.044 & 1.16 & 1.036 & \underline{ 1.035} & 1.487 & 1.136 & 1.036 & 1.304 & 1.187 \\ \midrule
 & SMAPE & \textbf{ 11.636} & 12.473 & \underline{ 11.807} & 13.474 & 12.84 & 11.829 & 13.639 & 12.369 & 12.84 & 12.78 & 12.909 \\
 & MASE & \textbf{ 1.556} & 1.677 & 1.59 & 1.866 & 1.701 & \underline{ 1.585} & 2.095 & 1.677 & 1.701 & 1.756 & 1.771 \\
\multirow{-3}{*}{\begin{tabular}[c]{@{}c@{}}Weighted\\      Average\end{tabular}} & OWA & \textbf{ 0.836} & 0.898 & \underline{ 0.851} & 0.985 & 0.918 & \underline{ 0.851} & \underline{ 0.851} & 0.894 & 0.918 & 0.93 & 0.939 \\ \midrule
\multicolumn{2}{c}{First Count} & \textbf{ 14} & 0 & \underline{ 1} & 0 & 0 & 0 & 0 & 0 & 0 & 0 & 0 \\ \bottomrule
\end{tabular}%
}
\end{table*}

\subsubsection{Learning objective.} To avoid the product being an arbitrarily large number that burdens the computation, we take the logarithm to $\text{vol}(\boldsymbol{S})$, which leads to the proposed diversification loss $\mathcal{L}_{dcs}$ for learning complementary sequence:
\begin{align}
    \mathcal{L}_{dcs}(\boldsymbol{S}) = - \sum_i^K 2\log ((\sigma_{\boldsymbol{S}} )_i+\epsilon),\quad\text{s.t. } \forall \|\boldsymbol{S}_i\|=1, 
\end{align}
where $\epsilon$ is a very small number to avoid the logarithmic operation being negative infinite, and apparently a smaller $\mathcal{L}_{dcs}(\boldsymbol{S})$ denotes a larger volume. Computing this loss only requires $\mathcal{O}(K^2P)$ due to the SVD decomposition, which is negligible. We compute the diversification loss independently for each channel in parallel to minimize the computational cost. The overall objective function is a weighted combination of MSE loss and the proposed diversification loss:
\begin{align}
   \mathcal{L}_{obj} =  \mathcal{L}_{mse}(\hat{\boldsymbol{Y}},\boldsymbol{Y})+\lambda_{dcs}\mathcal{L}_{dcs}(\boldsymbol{S}) ,
\end{align}
where $\lambda_{dcs}>0$ is the balance weight. This is a tunable hyper-parameter and we set it to 0.1 for all experiments. The overall algorithmic summary is presented in Appendix {C}.
% From the geometric perspective, it is easy to understand this fact. 

\section{Experiments and Results}

\paragraph{Experimental setup.} We evaluate our proposed architecture on both \textbf{long-term forecasting} task and \textbf{short-term forecasting} task with 8 datasets and 6 datasets, respectively. It is worth mentioning that each dataset in the long-term setting only contains one continuous time series, and the samples are obtained by sliding window, while M4 involves 100,000 different time series collected in different frequencies, which makes the forecasting on M4 become challenging \cite{wu2023timesnet}. For the metrics, we adopt the mean square error (MSE) and mean absolute error (MAE) for long-term forecasting, and following TimesNet \cite{wu2023timesnet}, the symmetric mean absolute percentage error (SMAPE), mean absolute scaled error (MASE) and overall weighted average (OWA) as the metrics for the short-term forecasting, where OWA is a special metric used in M4 competition.
The \textbf{lower} value implies a better performance for all metrics. A summary of the setup is shown in the right table. Please refer to Appendix {B} for details of the datasets and metrics. 
\begin{table}[]
\centering
% \caption{}
% \label{tab:dataset}
\resizebox{0.4\textwidth}{!}{%
\begin{tabular}{l|c|c|c|c}\toprule
Tasks & Benchmarks & Metrics & Dim &  Length \\ \midrule
Long-term & \begin{tabular}[c]{@{}c@{}}ETTh1, ETTh2, ETTm1, ETTm2,\\      Electricity, Traffic, Weather, Exchange\end{tabular} & MSE, MAE & 7$\sim$862 & 96$\sim $720 \\ \midrule
Short-term & \begin{tabular}[c]{@{}c@{}}M4 Yearly/Quarterly/Monthly/\\      Weekly/Daily/Hourly (6 subsets)\end{tabular} & SMAPE, MASE, OWA & 1 & 6$\sim$ 48 \\ \bottomrule
\end{tabular}%
}
\end{table}
% \begin{table}[hbp]
%   % \vspace{-5pt}
%   % \caption{Summary of experiment benchmarks.}\label{tab:benchmarks}
%   % \vskip 0.02in
%   \centering
%   \resizebox{1\linewidth}{!}{
%   \begin{small}
%   \renewcommand{\multirowsetup}{\centering}
%   \setlength{\tabcolsep}{6pt}
%   \begin{tabular}{c|l|c|c}
%     \toprule
%     \scalebox{0.95}{Tasks} & \scalebox{0.95}{Benchmarks} & \scalebox{0.95}{Metrics} & \scalebox{0.95}{Series Length} \\
%     \toprule
%     \scalebox{0.95}{\multirow{3}{*}{Forecasting}}  & \scalebox{0.95}{\textbf{Long-term}: ETTh1, ETTh2, ETTm1, ETTm2, } & \scalebox{0.95}{\multirow{2}{*}{MSE, MAE}} & \scalebox{0.95}{96$\sim$720} \\
%     & \scalebox{0.95}{Electricity, Traffic, Weather, Exchange} & & \scalebox{0.95}{(ILI: 24$\sim$60)}\\
%     \cmidrule{2-4}
%     & \scalebox{0.95}{\textbf{Short-term}: M4 (6 subsets)}  & \scalebox{0.95}{SMAPE, MASE, OWA} & \scalebox{0.95}{6$\sim$48} \\
%     \bottomrule
%     \end{tabular}
%     \end{small}
% }
%   % \vspace{-8pt}
% \end{table}
% \usepackage{graphicx}

% \subsection{Long-term Forecasting}

\setcounter{figure}{3}
\begin{figure*}[]
    \centering
    \includegraphics[width=1.0\textwidth]{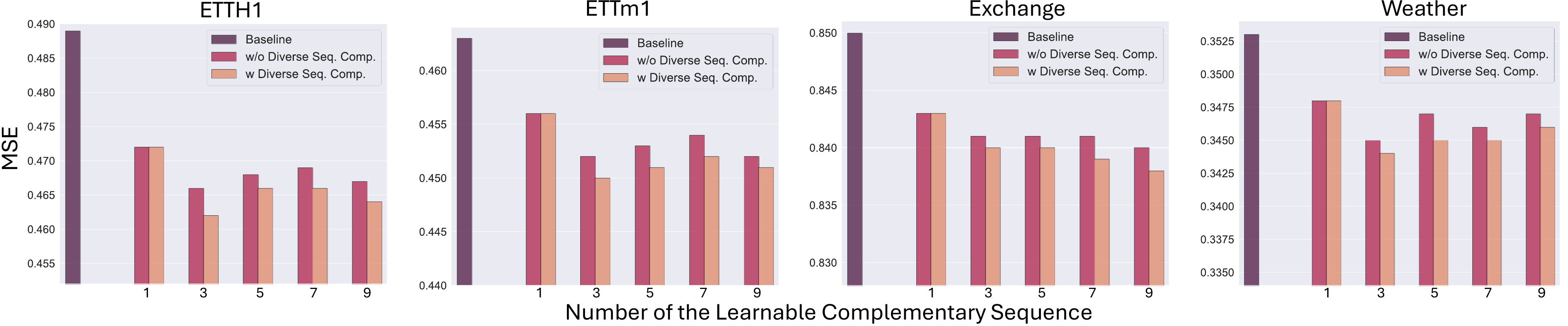}
    % \vspace{-0.2cm}
    \caption{Ablation studies on the number of learnable Sequence Complementors and the diversified Sequence Complementors on different datasets. The results suggest that when the number of complementors is equal to \textbf{3}, the overall performance is desired.  }
    % \vspace{-0.1cm}
    \label{fig:aba1}
\end{figure*}

Following \cite{wu2023timesnet, liu2024itransformer}, we fix the look-back window length to 96, and the forecast horizons are set as $\left\{96, 192, 336, 720 \right\}$. We include multiple recent transformer-based methods for long-term forecasting, such as PatchTST \cite{nie2023a}, iTransformer \cite{liu2024itransformer}, Crossformer \cite{zhang2023crossformer}, Fedformer \cite{zhou2022fedformer}, Autoformer \cite{wu2021autoformer}, Non-stationary Transformer \cite{liu2022non}.
We also include recent models with other architecture, such as TimesNet \cite{wu2023timesnet}, RLinear  \cite{li2023revisiting}, DLinear \cite{zeng2023transformers}, and SCINet \cite{liu2022scinet}. Since iTransformer is designed for learning cross-variate information, while short-term forecasting is univariate, hence we do not intend to include it in the comparison for short-term forecasting. In the implementation of our method, we use Adam optimizer \cite{kingma2014adam} and fix the learning rate to 0.0001. The number of the complementors is fixed at 3. Please refer to Appendix {C} for more details of the implementation. All results are reported over an average of 5 runs. The standard deviation and running time are reported in Appendix {D}, which confirms the robustness and efficiency of our method.

\paragraph{Main Results on Long-term Forecasting.} The quantitative results of the long-term time series forecasting are shown in Table \ref{tab:LONG}. The results demonstrate that the proposed model outperforms all other competing methods across most datasets and secures first place 47 times. In most other cases, our method can achieve second place, such as on ECL and Traffic datasets.
In comparison, the second-best method (i.e., iTransformer) achieves first place 19 times. Remarkably, Our method shows a significant improvement over PatchTST(\citeyear{nie2023a}) and wins it above 42 times across different settings and metrics. 
All of these results confirm the superiority of our method for long-term time series forecasting. The exemplary qualitative forecasting results are shown in Fig. \ref{fig:vis1}, where our method shows close coherence to the ground truth with minimal error, whereas some transformers (e.g., Non-stationary transformer and Crossformer) show disastrous alignment to the ground truth. We have additional visualization in Appendix {D}. Please also refer to Appendix {D} for the results on ILI dataset.
\setcounter{figure}{2}
\begin{figure}[!t]
    \centering
    \includegraphics[width=0.45\textwidth]{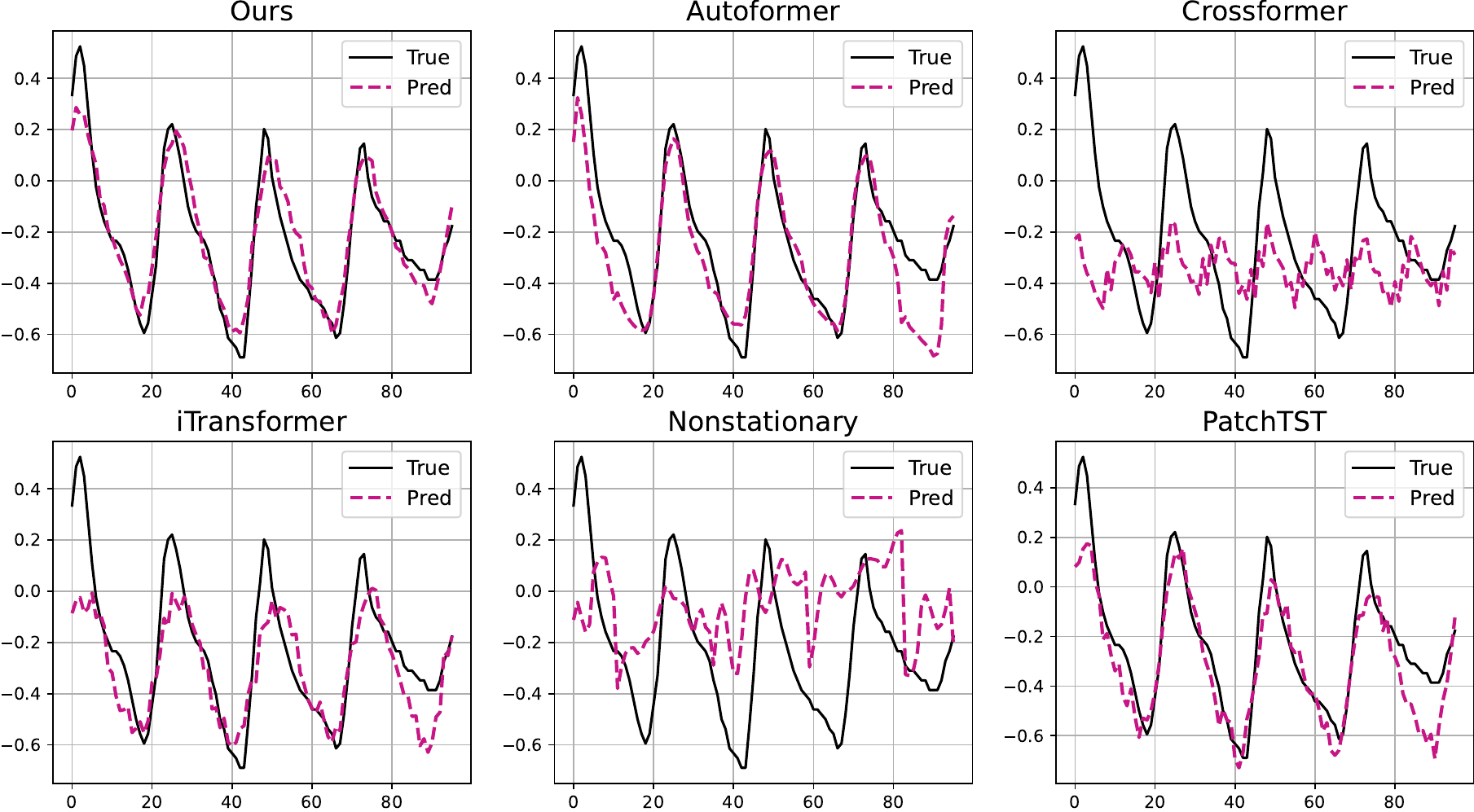}
    % \vspace{-0.2cm}
    \caption{The qualitative results on ETTh2 Dataset.}
    % \vspace{-0.1cm}
    \label{fig:vis1}
\end{figure}

\paragraph{Main Results on Short-term Forecasting.} The quantitative results for the short-term time series forecasting are shown in Table~\ref{tab:SHORT}. Different from the long-term forecasting tasks, the short-term forecasting tasks exhibit more temporal variations, as they are collected from different sources. 
The proposed still demonstrates superior performance on short-term forecasting tasks across different settings compared to other competing methods. Notably, our method secures the best performance 14 times out of all 15 settings. This suggests that the proposed method is even more effective on datasets with more temporal heterogeneity.

% There is a difference with long-term forecasting is that short-term time series are collected from different sources, which results in the temporal variations being heterogeneous~\cite{wu2023timesnet}. Even in such a more challenging task, the result still points out that our proposed method can consistently beat other transformer-based methods and achieve the overall best performance on SMAPE, MASE, and OWA.

\section{Analysis}
\paragraph{Effectiveness of Sequence Complementors.} We conduct the ablation analysis on Exchange, Weather, ETTm1, and ETTh1 datasets, where the baseline is PatchTST. As shown in Fig. \ref{fig:aba1}, it is evident that adding complementary sequences reduces the forecasting errors. Though there is some variability in performance when using a different number of complementary sequences ($K$), the improvement becomes substantial when $K \geq 3$. Therefore, we fix the number to 3 in our main experiments. For a Wilcoxon signed-rank test, please see Appendix {D}.

\paragraph{Effectiveness of Diversification.} We conduct the ablations on the diversification loss when $K \geq 3$, as it is meaningless to diversify only one complementary sequence. 
As shown in Fig. \ref{fig:aba1},  our model can benefit from this diversification loss with an additional reduction in forecasting error. 
We provide the similarity matrix for the learned complementary sequences with/without diversification loss in Appendix {D}.
% and obviously, the complementors learning with diversity loss have become almost orthogonal with each other, suggesting that the diversity is maximization. 

% \begin{figure}[htbp]
%     \centering
%     \includegraphics[width=0.3\textwidth]{figs/DIV.jpg}
%     % \vspace{-0.2cm}
%     \caption{}
%     % \vspace{-0.1cm}
%     \label{fig:DIV}
% \end{figure}

% \begin{figure}[htbp]
%     \centering
%     \includegraphics[width=0.3\textwidth]{figs/training_dynamics.png}
%     % \vspace{-0.2cm}
%     \caption{}
%     % \vspace{-0.1cm}
%     \label{fig:training_dynamics}
% \end{figure}

\paragraph{Case Study on Training Dynamics.}  
We conduct a case study on how the richness of the latent representations evolves during training (see Fig. \ref{fig:training_dynamics}). We observe that as training progresses, the richness of the learned representations increases, indicating that the model is learning more complex and diverse features. The curve might plateau once the model reaches a certain level of feature richness, suggesting convergence in feature learning. The main observation is that our model can consistently learn richer representations than the baseline and translate them to a lower forecasting error without slowing the model convergence. This again highlights the effectiveness of complementary sequences.

\setcounter{figure}{4}
\begin{figure}[!t]
    \centering
   \includegraphics[width=0.35\textwidth]{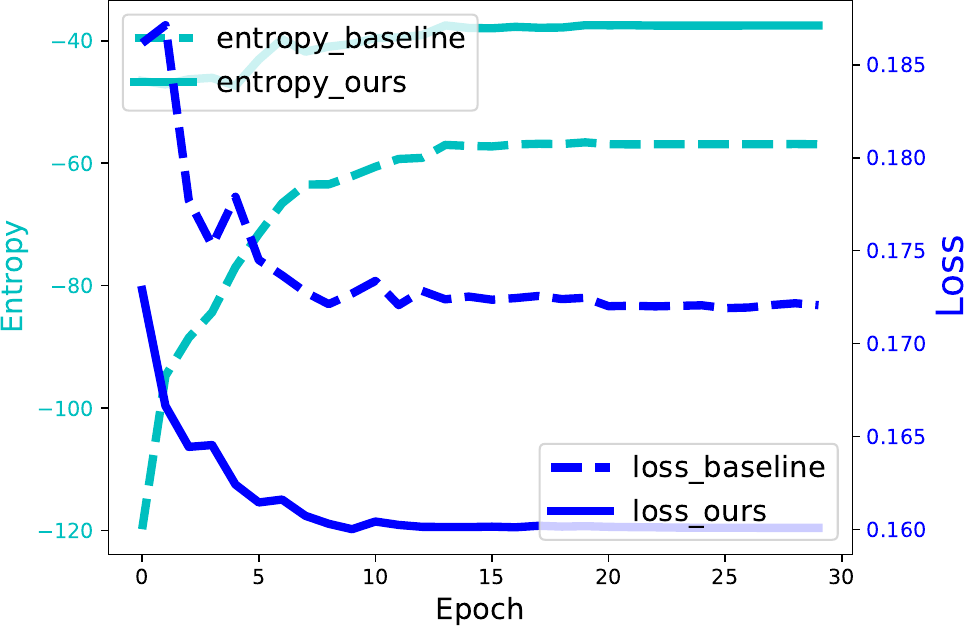}
    % \vspace{-0.2cm}
    \caption{The comparison of training dynamics with and without Sequence Complementors.}
    % \vspace{-0.1cm}
    \label{fig:training_dynamics}
\end{figure}

\paragraph{Case Study on iTransformer.} We conduct a case study on iTransformer to show that our method can also be applied to other transformers. 
% It is evident that iTransformer performs well on datasets with a large number of variates; hence, we are interested in the possibility of improving it on datasets with limited variates. 
% This is due to the recent methods that have approached the performance ceiling, so even a slight improvement over previous SOTA methods is even challenging, as 
% is highly significant.
Though iTransformer has achieved good performance for long-term time series forecasting tasks (see Table~\ref{tab:LONG}), its performance is further boosted after integrating the proposed Sequence Complementor (see Table~\ref{tab:case-study}). In particular, the integration of Sequence Complementor results in an average reduction of MSE and MAE by 2.4\% and 1.22\%,  respectively. The larger reduction rate of MSE over MAE matches our theoretical analysis that the integration of Sequence Complementor is likely to lower the bound of MSE (Theorem~\ref{thm:1}).
We also observe that the learned representations are enriched after the integration of our method into iTransoformer (for details, see Appendix D). These results suggest that the effectiveness of the proposed method is model-agnostic.

\begin{table}[!t]
\centering
\caption{The case study on iTransformer, where we add 3 complementary sequences to the patchified features in iTransformer. The performance is averaged over four prediction lengths. Please refer to Appendix {D} for full results.}
\label{tab:case-study}
\resizebox{0.4\textwidth}{!}{%
\begin{tabular}{l|cccccc}\toprule
Dataset & \multicolumn{2}{c}{ETTm1} & \multicolumn{2}{c}{ETTm2} & \multicolumn{2}{c}{ETTh1} \\
Metric & MSE & MAE & MSE & MAE & MSE & MAE \\ \midrule
\ \ \ iTransformer  & 0.407 & 0.41 & 0.288 & 0.332 & 0.454 & 0.447 \\
+ Seq. Comp.  & \textbf{0.393} & \textbf{0.402} & \textbf{0.285} & \textbf{0.329} & \textbf{0.443} & \textbf{0.441} \\
\ \ \ Reduction (\%) & 3.5\% & 2.0\% & 0.9\% &  0.8\% & 2.5\% & 1.4\% \\ \midrule
Dataset & \multicolumn{2}{c}{ETTh2} & \multicolumn{2}{c}{Exchange} & \multicolumn{2}{c}{Weather} \\
Metric & MSE & MAE & MSE & MAE & MSE & MAE \\ \midrule
\ \ \ iTransformer   & 0.383 & 0.407 & 0.36 & 0.403 & 0.258 & 0.278 \\
+ Seq. Comp.  & \textbf{0.376} & \textbf{0.403} & \textbf{0.346} & \textbf{0.398} & \textbf{0.254} & \textbf{0.276} \\
\ \ \ Reduction  (\%) & 1.9\% & 1.1\% & 3.8\% & 1.3\% & 1.7\% & 0.7\% \\ \bottomrule
\end{tabular}%
}
\end{table}

\section{Related Work}
Despite there being other pillars of models in time series forecasting, transformers have been one of the main streams since its introduction. This is because the self-attention in transformers is effective in capturing long-term temporal dependencies and complex multivariate correlations compared to other competing models, such as RNNs~\cite{lai2018modeling,patro2024simba}, MLPs~\cite{chen2023tsmixer,wang2024timemixer}, and CNNs~\cite{wu2023timesnet,donghao2024moderntcn,liu2022scinet,wu2023timesnet}. 
 However, the direct adoption of vanilla transformers faces several challenges in TSF, such as inadequate tokenization of time series sequences, inefficient processing of long sequences, difficulties in series rationalization, and insufficient modeling of variable and channel mutual information. To address these, many follow-up variants have been introduced for TSF. PatchTST \cite{nie2023a} addresses these problems by segmenting each channel/univariate time sequence into sub-level patchified sequences, designing a tokenization module inspired by ViT~\cite{vit}.
 iTransformer~\cite{liu2024itransformer} and Crossformer~\cite{zhang2023crossformer} further redesign the self-attention and feed-forward network to better capture cross-variate dependencies. Non-stationary Transformer \cite{liu2022non} introduces a simple yet effective self-attention for time-series stationarization, enhancing predictive capability for non-stationary series without adding extra parameters.  Fedformer \cite{zhou2022fedformer} and Autoformer \cite{wu2021autoformer} introduce frequency-enhanced and auto-correlation decomposition into vanilla transformers to efficiently handle the long-range time series with complex patterns. 

% Although empirical studies, such as \cite{zeng2023transformers,lin2024sparsetsf}\footnote{A detailed discussion about the claim of \cite{zeng2023transformers} is in Appendix E}, have raised challenges about the effectiveness of transformers in TSF, we argue that self-attention remains an effective mechanism for TSF modeling by nature. 
Although empirical studies such as \cite{zeng2023transformers} have questioned the effectiveness of transformers in TSF\footnote{A detailed discussion about the claim in \cite{zeng2023transformers} is provided in Appendix E.}, we hold the view that self-attention has untapped potentials for time series forecasting due to the modeling of long-range dependencies.
Our investigation suggests that the limitations of transformers in TSF stem from their inability to learn rich and generalizable feature representations with limited data. To address this, we propose diversified complementary sequences rather than altering architectural designs, as seen in~\cite{nie2023a, zhang2023crossformer, liu2024itransformer, liu2022non, zhou2022fedformer, wu2021autoformer}. Our method is orthogonal and can be integrated with these previous models.

\section{Conclusion}
In this paper, we investigate the potential reason that leads transformer-based methods with various performances. First, we conduct experiments on multiple recent transformer-based methods and measure the richness of their representation by rank measures and an information-theoretical measure. We show the interesting finding that a richer representation can often translate to a better forecasting performance. Based on this finding and guided by information-theoretical knowledge, we propose the \texttt{Sequence Complementors}, which can enhance the representation and be seamlessly integrated into the Transformer-based framework.
To further strengthen the complementors, we propose a differentiable volume maximization loss. The empirical results on 8 long-term forecasting datasets and 6 short-term forecasting datasets confirm the superiority of our proposed method. We hope this work provides new insights into understanding and designing transformers for time series forecasting.

% \newpage

\section*{Acknowledgment}
This material is based upon the work supported by the National Science Foundation under Grant Number 2204721 and partially supported by our collaborative project with MIT Lincoln Lab under Grant Number 7000612889.

\bibliography{aaai25}

\newpage
\section*{Appendix}
\maketitle

\appendix
\setcounter{secnumdepth}{1}

\renewcommand{\thefigure}{S\arabic{figure}}
\renewcommand{\thetable}{S\arabic{table}}

\section{Proofs}

\subsection{Proof of Justification 1}
\begin{align}
    H(\boldsymbol{Z}_{enc})\leq  H(\boldsymbol{Z}_{enc},\boldsymbol{S}).
\end{align}

\begin{proof}
We can decompose the joint entropy $H(\boldsymbol{Z}_{enc},\boldsymbol{S})$ to,
\begin{align}
    H(\boldsymbol{Z}_{enc},\boldsymbol{S}) = H(\boldsymbol{Z}_{enc})+H(\boldsymbol{S}|\boldsymbol{Z}_{enc}).
\end{align}
Since the conditional entropy $H(\boldsymbol{S}|\boldsymbol{Z}_{enc})$ is always non-negative \cite{cover1999elements}, the proof is completed.

\end{proof}
\subsection{Proof of Lemma 1}
\begin{lemma}%\label{lemma:1}
    % Considering a linear projector as the decoder, i.e. $\hat{\boldsymbol{Y}}= \texttt{reshape}(\text{Vec}({\boldsymbol{H}_{enc}})\boldsymbol{W}_{enc})$ and u
    Under Gaussian assumption, the minimum mean-squared error (MMSE), is bounded by,  
   \begin{align}
       \text{MMSE} \overset{}{\geq} \frac{\exp2 H(\boldsymbol{Y} |\boldsymbol{Z}_{enc} )}{2\pi e} 
       .
        \end{align}
        Here,  $H(\cdot|\cdot)$ denotes the conditional entropy. 
\end{lemma}

\begin{proof}
We first construct the chain $\boldsymbol{Z}_{enc}$ similar to Eq. 1 in the main manuscript,
\begin{align}\label{eq:trans0}
     \cdots\xrightarrow{}\boldsymbol{Z}_{enc}  \xrightarrow{f_\texttt{Dec}}  \boldsymbol{Z}_{dec} 
   \xrightarrow{f_\texttt{linear}}\hat{\boldsymbol{Y}},
\end{align}
Here, we take out the last linear layer separately for the convenience of proving the lemma.
Now, we can rewrite the propagation in the last layer as,
\begin{align}
    \hat{\boldsymbol{Y}} =\texttt{reshape}(\text{vec}(\boldsymbol{Z}_{dec} )\boldsymbol{W}).
\end{align}
Under Gaussian assumption, i.e. both $\boldsymbol{Z}_{dec} $, $ \hat{\boldsymbol{Y}}$  , and $\boldsymbol{Y}$ obey Gaussian distribution, according to \cite{carson2012communications,prasad2010certain}, we can obtain,
\begin{align}\label{eq:l1}
      \text{MMSE} \overset{}{\geq} \frac{\exp2 H(\boldsymbol{Y} |\boldsymbol{Z}_{dec} )}{2\pi e}.
\end{align}

We observe $\boldsymbol{Z}_{dec}=f_\texttt{dec}(\boldsymbol{Z}_{enc})$, due to the chain property \cite{cover1999elements}, $\boldsymbol{Z}_{enc}$ should always contains equal or more information of $\boldsymbol{Z}_{dec}$, hence knowing $\boldsymbol{Z}_{enc}$ can further reduce the uncertainty about $\boldsymbol{Y}$,
\begin{align}\label{eq:l2}
H(\boldsymbol{Y} |\boldsymbol{Z}_{dec} ) \geq H(\boldsymbol{Y} |\boldsymbol{Z}_{enc}).
\end{align}
Then, substitute Eq. \ref{eq:l2} to Eq. \ref{eq:l1}, the proof is completed.
 
\end{proof}

\subsection{Proof of Theorem 1}
\begin{theorem}
    The integration of the proposed complementary sequences lowers the bound of MMSE:
    \begin{align}
    H(\boldsymbol{Y} |\boldsymbol{Z}_{{enc}},\boldsymbol{S})\leq  H(\boldsymbol{Y} |\boldsymbol{Z}_{{enc}}).
    \end{align}  
\end{theorem}

\begin{proof}
This is trivial as conditional information is monotonicity, meaning conditioning on additional information (in this case, the complementary sequences $\boldsymbol{S}$) is impossible to increase the uncertainty about $\boldsymbol{Y}$.
\end{proof}

\subsection{Proof of Theorem 2}
\begin{theorem}
    If $\forall \|\boldsymbol{S}_i\|=1, i\in [K]$ holds, maximizing $\text{vol}(\boldsymbol{S})$ results in $\boldsymbol{S}_i \perp \boldsymbol{S}_j,$ for all $  \forall i \neq j,i,j \in [K]$.
\end{theorem}
\begin{proof}
We can prove the theorem by the arithmetic mean-geometric mean (AM-GM) inequality \cite{steele2004cauchy}.
 We  rewrite $ \prod_{i=1}^K (\sigma_{\boldsymbol{S}} )_i=\sqrt{\operatorname{det}(\boldsymbol{S}\boldsymbol{S}^T)}$, where the strictly positive singular values of $(\boldsymbol{S}\boldsymbol{S}^T)$ are $\{(\sigma_{\boldsymbol{S}} )_i^2\}_{i=1}^K$.
Applying AM-GM inequality to $(\sigma_{\boldsymbol{S}} )_i^2$,
\begin{align}\label{eq:p1}
    \frac{\sum_i^K(\sigma_{\boldsymbol{S}} )_i^2}{K} \geq \sqrt[K]{ \cdot \prod_i^K(\sigma_{\boldsymbol{S}} )_i^2}.
\end{align}
Since $\|\boldsymbol{S}_i\|=1$, we have,
\begin{align}\label{eq:p2}
 \sqrt{K} =  \|\boldsymbol{S}\|_F = \sqrt{\text{Tr}(\boldsymbol{S}\boldsymbol{S}^T))} = \sqrt{\sum_i^K(\sigma_{\boldsymbol{S}} )_i^2}.
\end{align}
Hence, $\sum_i^K(\sigma_{\boldsymbol{S}} )_i^2=K$. Substitute it to Eq. \ref{eq:p1},
\begin{align}
    1\geq \sqrt[K]{ \cdot \prod_i^K(\sigma_{\boldsymbol{S}} )_i^2} = \sqrt[K]{\text{vol}(\boldsymbol{S})^2}.
\end{align}
According to AM-GM inequality, the equality of Eq. \ref{eq:p1} holds if and only if $(\sigma_{\boldsymbol{S}} )_1^2=\cdots=(\sigma_{\boldsymbol{S}} )_K^2 = \frac{K}{K}=1$. At this situation, $\boldsymbol{S}\boldsymbol{S}^T=\boldsymbol{I}_k$, where $\boldsymbol{I}_k$ is the identity matrix with shape $K\times K$. Recap that each element of $(\boldsymbol{S}\boldsymbol{S}^T)$ denotes the cosine similarity between $\boldsymbol{S}_j,\boldsymbol{S}_i$, hence, when the $\text{vol}(\boldsymbol{S})$ is maximized, for every $i \neq j,i,j \in [K]$, $\boldsymbol{S}_i\boldsymbol{S}_j^\top = 0$ implies $\boldsymbol{S}_i$ and $\boldsymbol{S}_j$ are orthogonal.
\end{proof}

\section{Experiment Setup}
\begin{table*}[!t]
  % \vspace{-20pt}
  \caption{Dataset descriptions. The dataset size is organized in (Train, Validation, Test).}\label{tab:dataset}
  % \vskip 0.05in
  \centering
  \begin{threeparttable}
  \begin{small}
  \renewcommand{\multirowsetup}{\centering}
  \setlength{\tabcolsep}{4pt}
  \begin{tabular}{c|l|c|c|c|c}
    \toprule
    Tasks & Dataset & Dim & Series Length & Dataset Size & \scalebox{0.8}{Information (Frequency)} \\
    \toprule
     & ETTm1, ETTm2 & 7 & \scalebox{0.8}{\{96, 192, 336, 720\}} & (34465, 11521, 11521) & \scalebox{0.8}{Electricity (15 mins)}\\
    \cmidrule{2-6}
     & ETTh1, ETTh2 & 7 & \scalebox{0.8}{\{96, 192, 336, 720\}} & (8545, 2881, 2881) & \scalebox{0.8}{Electricity (15 mins)} \\
    \cmidrule{2-6}
    Forecasting & Electricity & 321 & \scalebox{0.8}{\{96, 192, 336, 720\}} & (18317, 2633, 5261) & \scalebox{0.8}{Electricity (Hourly)} \\
    \cmidrule{2-6}
    (Long-term) & Traffic & 862 & \scalebox{0.8}{\{96, 192, 336, 720\}} & (12185, 1757, 3509) & \scalebox{0.8}{Transportation (Hourly)} \\
    \cmidrule{2-6}
     & Weather & 21 & \scalebox{0.8}{\{96, 192, 336, 720\}} & (36792, 5271, 10540) & \scalebox{0.8}{Weather (10 mins)} \\
    \cmidrule{2-6}
     & Exchange & 8 & \scalebox{0.8}{\{96, 192, 336, 720\}} & (5120, 665, 1422) & \scalebox{0.8}{Exchange rate (Daily)}\\
    \cmidrule{2-6}
     & ILI & 7 & \scalebox{0.8}{\{24, 36, 48, 60\}} & (617, 74, 170) & \scalebox{0.8}{Illness (Weekly)} \\
    \midrule
     & M4-Yearly & 1 & 6 & (23000, 0, 23000) & \scalebox{0.8}{Demographic} \\
    \cmidrule{2-5}
     & M4-Quarterly & 1 & 8 & (24000, 0, 24000) & \scalebox{0.8}{Finance} \\
    \cmidrule{2-5}
    Forecasting & M4-Monthly & 1 & 18 & (48000, 0, 48000) & \scalebox{0.8}{Industry} \\
    \cmidrule{2-5}
    (short-term) & M4-Weakly & 1 & 13 & (359, 0, 359) & \scalebox{0.8}{Macro} \\
    \cmidrule{2-5}
     & M4-Daily & 1 & 14 & (4227, 0, 4227) & \scalebox{0.8}{Micro} \\
    \cmidrule{2-5}
     & M4-Hourly & 1 &48 & (414, 0, 414) & \scalebox{0.8}{Other} \\
    \bottomrule
    \end{tabular}
    \end{small}
  \end{threeparttable}
  % \vspace{-5pt}
\end{table*}
\subsection{Datasets}
Here, we adopt the description of datasets from \cite{wu2023timesnet} as shown in Table \ref{tab:dataset} below. These train/val/test splits are adhered to previous work.

\subsection{Metrics}

We skip reviewing MSE AND MAE here as they are very commonly used. For short-term forecasting, we adopt the symmetric mean absolute percentage error (SMAPE), mean absolute scaled error (MASE) and overall weighted average (OWA) as the metrics. It is worth mentioning that  OWA is a special metric used in M4 competition. They are presented as follow:
\begin{align*} \label{equ:metrics}
    \text{SMAPE} &= \frac{200}{T_{out}} \sum_{i=1}^{T_{out}} \frac{|\boldsymbol{Y}_{i} - \widehat{\boldsymbol{Y}}_{i}|}{|\boldsymbol{Y}_{i}| + |\widehat{\boldsymbol{Y}}_{i}|}, \\
    \text{MAPE} &= \frac{100}{T_{out}} \sum_{i=1}^{T_{out}} \frac{|\boldsymbol{Y}_{i} - \widehat{\boldsymbol{Y}}_{i}|}{|\boldsymbol{Y}_{i}|}, \\
    \text{MASE} &= \frac{1}{T_{out}} \sum_{i=1}^{T_{out}} \frac{|\boldsymbol{Y}_{i} - \widehat{\boldsymbol{Y}}_{i}|}{\frac{1}{{T_{out}}-m}\sum_{j=m+1}^{{T_{out}}}|\boldsymbol{Y}_j - \boldsymbol{Y}_{j-m}|}, \\
    \text{OWA} &= \frac{1}{2} \left[ \frac{\text{SMAPE}}{\text{SMAPE}_{\textrm{Naïve2}}}  + \frac{\text{MASE}}{\text{MASE}_{\textrm{Naïve2}}}  \right],
\end{align*}
where $m$ is the periodicity of the data. $\boldsymbol{Y},\widehat{\boldsymbol{Y}}\in\mathbb{R}^{T_{out}\times N}$ are the ground truth and prediction results of the future with $T_{out}$ time pints and $N$ dimensions. $\boldsymbol{Y}_{i}$ means the $i$-th future time point.

\section{Implementation Detail}
\subsection{Configuration of Main Experiment}
For the results of comparison methods, we adopt the values reported by iTransformer \cite{liu2024itransformer} for the long-term forecasting and TimesNet \cite{wu2023timesnet}  for the short-term forecasting. We have run their code and found the results are reproducible. 

We implemented all experiments on a cluster node with NVIDIA V100 (16 GB). We use Pytorch Library \cite{paszke2019pytorch} with version of 1.13. We use Adam optimizer \cite{kingma2014adam} and fix the learning rate to 0.0001. The other parameter is the default, as presented in Pytorch. The data pre-processing script, training framework, and module (e.g, transformer block), are built upon \url{https://github.com/thuml/Time-Series-Library}. The number of the complementors is fixed to 3. For the long-term forecasting task, the embedding dimension is set to \{128, 512\}.  The number of encoder blocks is searched from \{1,2,3\}, and the number of heads in each transformer block is fixed to 4. Having said that, we make consistent configurations for each dataset with different prediction lengths. All results are reported over an average of 5 runs.
For the long-term forecasting task, the mini-batch size is fixed to 16, the number of encoder blocks is fixed to 2, and the embedding dimension is set to 512. We also applied RevIN \cite{kim2022reversible} aligned with previous works.

\subsection{Configuration of the Analysis of Transformers}

Our primary investigation was conducted on the \textbf{Weather} dataset, where we fixed the dimensions of features to 128. We train these models with default parameters and collect the representations of 640 samples from each model after training the network well. In section \ref{appdenx:d}, we present an additional investigation on the ETTm1 dataset, which illustrates a similar conclusion. 

\newpage
\subsection{Algorithmic Summary (Channel-independent)}

\begin{algorithm}[]
\caption{Forward Propagation w/ Comp. (one-channel)}
\begin{algorithmic}[1]
\REQUIRE Input Sequence ${\boldsymbol{X}}[:,i]$, Complementors ${\boldsymbol{S}}$, model parameters: $f_\texttt{linear}, f_{\texttt{enc}},  f_{\texttt{enc}}$.
\ENSURE Forecasting data $\hat{\boldsymbol{Y}}$.
\STATE Patchify: $  \Tilde{\boldsymbol{Z}}_0\in\mathbb{R}^{ (L_c) \times P}\leftarrow \texttt{unrolling}( {\boldsymbol{X}}[:,i])$.
\STATE Add Comp.: $\Tilde{\boldsymbol{Z}}_0\in\mathbb{R}^{ (L_c+K) \times P} \leftarrow \texttt{concat}(\boldsymbol{Z}_0,\boldsymbol{S})$.
\STATE Embedding: $ \boldsymbol{Z}_{embed}\in\mathbb{R}^{ (L_c+K) \times D } \leftarrow f_\texttt{linear}(\Tilde{\boldsymbol{Z}}_0)$.
\STATE Pos. Encoding : $f_{\texttt{pos}}(\boldsymbol{Z}_{embed}[:L_c])+\boldsymbol{Z}_{embed}[:L_c]$, $\boldsymbol{Z}_{embed} \leftarrow \texttt{concat}( \boldsymbol{Z}_{embed}[:L_c],\boldsymbol{Z}_{embed}[L_c+1:])$.
\STATE Encoding: $\boldsymbol{Z}_{enc}\in\mathbb{R}^{ (L_c+K) \times D }\leftarrow f_{\texttt{enc}}(\boldsymbol{Z}_{embed})$.
\STATE Decoding: $\hat{\boldsymbol{Y}}= f_\texttt{dec}(\boldsymbol{Z}_{enc}[:L_c])$.

\end{algorithmic}
\label{alg:1}
\end{algorithm}

\subsection{Integrating Sequence Complementors into iTransformer  (iTrans.)}
Here, we provide an algorithmic summary of how to integrate the proposed Sequence Complementor into iTransformer.

\begin{algorithm}[]
\caption{Forward Propagation w/ Comp. (iTrans.)}
\begin{algorithmic}[1]
\REQUIRE Input Sequence ${\boldsymbol{X}}\in \mathbb{R}^{T_{in} \times N}$, Complementors ${\boldsymbol{S}}$, model parameters: $f_\texttt{linear}, f_{\texttt{enc}},  f_{\texttt{enc}}$.
\ENSURE Forecasting data $\hat{\boldsymbol{Y}}$.
\STATE Invert: $  \Tilde{\boldsymbol{Z}}_0\in\mathbb{R}^{ N \times T_{in}}\leftarrow \boldsymbol{X}.\texttt{transpose}$.
\STATE Add Comp.: $\Tilde{\boldsymbol{Z}}_0\in\mathbb{R}^{ (N+K) \times T_{in}} \leftarrow \texttt{concat}(\boldsymbol{Z}_0,\boldsymbol{S})$.
\STATE Embedding: $ \boldsymbol{Z}_{embed}\in\mathbb{R}^{ (N+K) \times D } \leftarrow f_\texttt{linear}(\Tilde{\boldsymbol{Z}}_0)$.
\STATE Encoding: $\boldsymbol{Z}_{enc}\in\mathbb{R}^{ (N+K) \times D }\leftarrow f_{\texttt{enc}}(\boldsymbol{Z}_{embed})$.
\STATE Decoding: $\hat{\boldsymbol{Y}}= \texttt{reshape}(f_\texttt{dec}(\boldsymbol{Z}_{enc}[:L_c]))$.

\end{algorithmic}
\label{alg:2}
\end{algorithm}

% \paragraph{Transformer.}
\section{Additional Results}

\subsection{Additional Investigation}\label{appdenx:d}
Here, we present the investigation, including iTransformer for both the Weather and ETTm1 datasets in Fig. \ref{fig:appendix}. The results support our observation that richer learned representations are associated with lower forecasting errors. The Pearson correlation test yielded a high statistical negative correlation (i.e., the absolute value of the correlation coefficient greater than 0.8 and the p-value less than 0.05).

\begin{figure*}[!t]
    \centering
    \includegraphics[width=0.87\textwidth]{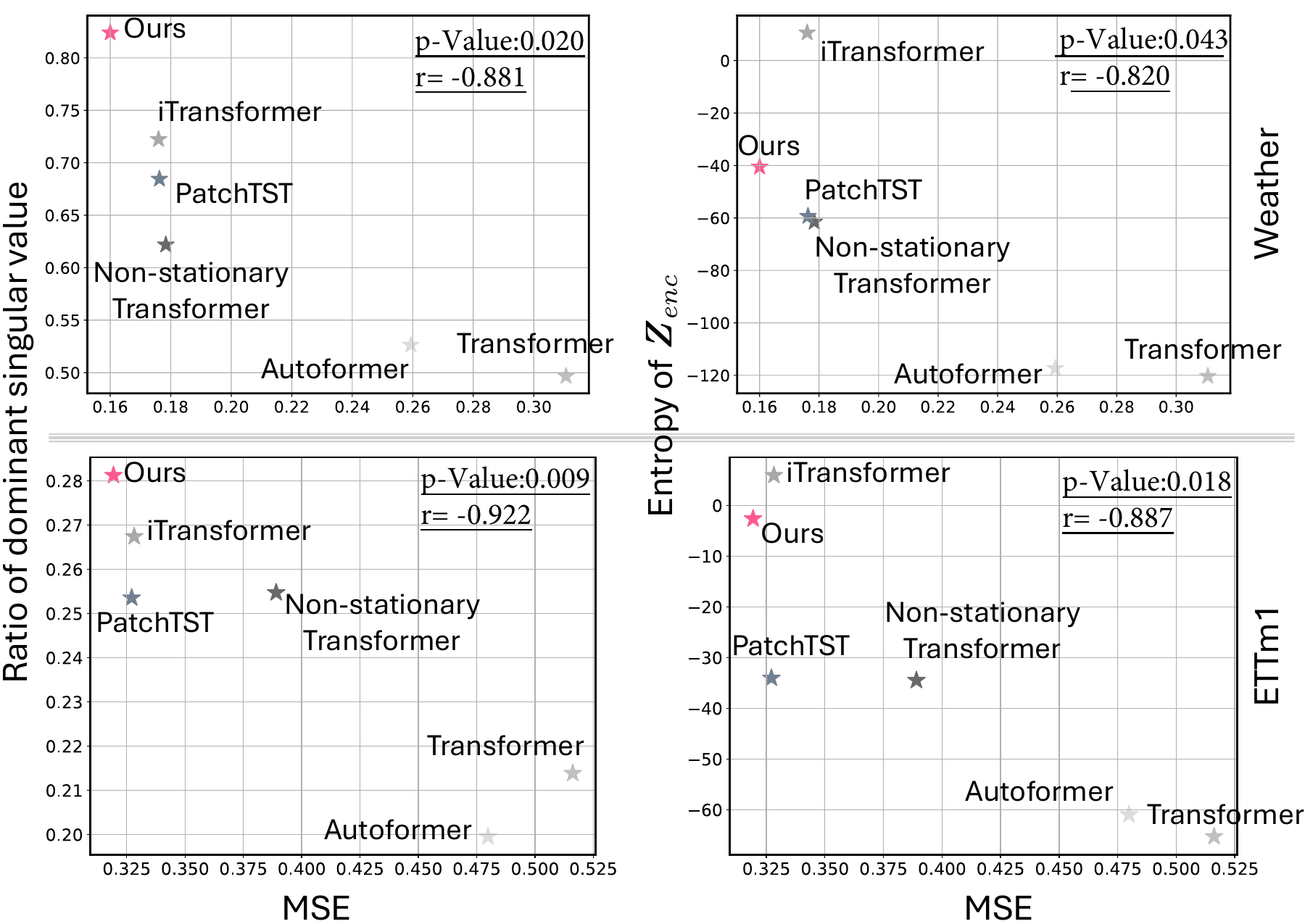}
    \caption{The investigation of the relations between learned representation and forecasting error by different Transofmer-based models. The Pearson correlation coefficient 
    $r$ quantifies the relationship between MSE and our two measures of feature richness. A value close to -1 indicates a strong negative correlation between these two quantities (p-Value < 0.05; a statistically significant correlation). These results support our claim that the richness of learned latent representations is negatively correlated to MSE.}
    \label{fig:appendix}
\end{figure*}

\subsection{Standard Deviation}
We reported the standard deviation on the ETT, exchange, and weather datasets, as shown in Table \ref{tab:sd}, where the standard deviation is at the 1e-4 or 1e-5 level, which is reasonably small. 
% \usepackage{graphicx}
\begin{table}[!t]
\centering
\caption{The standard deviation of our methods.}
\label{tab:sd}
\resizebox{0.45\textwidth}{!}{%
\begin{tabular}{cccccc} \toprule
\multicolumn{2}{c}{ETTh1} & \multicolumn{2}{c}{ETTh2} & \multicolumn{2}{c}{ETTm1} \\ \midrule
MSE & MAE & MSE & MAE & MSE & MAE \\ \midrule
7.33E-05 & 4.72E-04 & 5.56E-06 & 5.39E-04 & 1.57E-04 & 5.63E-04 \\
2.70E-04 & 3.03E-04 & 2.72E-05 & 6.58E-04 & 1.06E-05 & 6.48E-04 \\
2.81E-04 & 4.87E-04 & 1.87E-04 & 2.97E-04 & 4.19E-04 & 8.54E-04 \\
5.73E-04 & 3.45E-03 & 1.24E-04 & 5.22E-04 & 7.00E-05 & 6.36E-04 \\ \midrule
\multicolumn{2}{c}{ETTm2} & \multicolumn{2}{c}{Exchange} & \multicolumn{2}{c}{weather} \\ \midrule
MSE & MAE & MSE & MAE & MSE & MAE \\ \midrule
4.79E-05 & 4.09E-04 & 2.18E-05 & 1.37E-04 & 5.81E-05 & 2.05E-04 \\
1.05E-04 & 4.69E-05 & 2.53E-06 & 3.13E-05 & 1.99E-04 & 7.92E-05 \\
3.88E-04 & 8.97E-04 & 2.98E-04 & 4.56E-04 & 3.58E-04 & 6.34E-04 \\
1.56E-04 & 8.84E-04 & 7.18E-04 & 5.59E-04 & 1.05E-03 & 9.11E-04 \\ \bottomrule
\end{tabular}%
}
\end{table}

\subsection{Running Time}
We present the running time in Table \ref{tab:runtime}, which is reasonable as around only 10\% additional time cost in ECL dataset.

% \usepackage{multirow}
% \usepackage{graphicx}
\begin{table}[]
\centering
\caption{The running time per 100 iterations. }
\label{tab:runtime}
\resizebox{0.45\textwidth}{!}{%
\begin{tabular}{cccccc}\toprule
\multicolumn{1}{l}{Dataset} & Prediction Length & 96 & 192 & 336 & 720 \\ \midrule
\multirow{3}{*}{ECL} & Baseline & 0.6647 & 0.7318 & 0.8907 & 1.1477 \\
 & Ours & 0.7507 & 0.7929 & 0.9443 & 1.1981 \\
 & Increment & +12.94\% & +8.35\% & +6.02\% & +4.39\% \\ 
%  \midrule
% \multirow{3}{*}{ETTm2} & Baseline & 0.0105 & 0.011 & 0.0114 & 0.0125 \\
%  & Ours & 0.0130 & 0.0132 & 0.0136 & 0.0149 \\
%  & Increment & +23.81\% & +20.00\% & +19.30\% & +19.20\% \\
 \bottomrule
\end{tabular}%
}
\end{table}

\subsection{Wilcoxon signed-rank test}
the Wilcoxon signed-rank test is a non-parametric statistical test, which is typically used as an alternative to the paired t-test when the data cannot be assumed to be normally distributed. As suggested by \cite{demvsar2006statistical}, we can use it to compare two models across different datasets. Here, we construct the test between our method and the baseline on all datasets. As shown in Table \ref{tab:stest}, the p-value is tiny and much smaller than 0.05, which confirms the statistical superiority of our method.

 % The test will tell you if there's a statistically significant difference between the classifiers' performance across the datasets, without assuming that the performance differences are normally distributed.

\begin{table}[!t]
    \centering
      \caption{The results of Wilcoxon signed-rank test on both MSE and MAE.}
      \label{tab:stest}
    \resizebox{0.25\textwidth}{!}{%
    \begin{tabular}{c|cc} \toprule
         & MSE
 & MSE\\ \midrule
     Test Statistic  & 26
 & 35.5
 \\
p-value &3.77E-07
  &  7.47E-07
\\ \bottomrule
    \end{tabular}
    }

    \label{tab:my_label}
\end{table}

\subsection{Additional Results on Illness}
Since the setup on Illness (ILI) differs from the setups of other long-term forecasting datasets, we present its results here. The results shown in Table \ref{tab:ili} imply a consistent superiority of our method.

% \usepackage{graphicx}
\begin{table*}[!t]
\centering
\caption{Forecasting resulsts on ILI dataset. The best results are in \textbf{bold}.}
\label{tab:ili}
\resizebox{0.92\textwidth}{!}{%
\begin{tabular}{c|cccccccccccccccccc}\toprule
 & \multicolumn{2}{c}{Ours} & \multicolumn{2}{c}{iTransformer} & \multicolumn{2}{c}{PatchTST} & \multicolumn{2}{c}{Crossformer} & \multicolumn{2}{c}{TimesNet} & \multicolumn{2}{c}{Dlinear} & \multicolumn{2}{c}{FEDformer} & \multicolumn{2}{c}{Stationary} & \multicolumn{2}{c}{Autoformer} \\ 
& MSE & MAE & MSE & MAE & MSE & MAE & MSE & MAE & MSE & MAE & MSE & MAE & MSE & MAE & MSE & MAE & MSE & MAE \\ \midrule
24 & \textbf{2.175} & \textbf{0.854} & 2.186 & 0.931 & 2.234 & 0.891 & 3.461 & 1.237 & 2.317 & 0.934 & 2.398 & 1.040 & 3.228 & 1.260 & 2.294 & 0.945 & 3.483 & 1.287 \\
36 & 2.342 & 0.906 & 1.860 & 0.864 & 2.316 & 0.932 & 3.762 & 2.175 & 1.972 & 0.920 & 2.646 & 1.088 & 2.679 & 1.080 & \textbf{1.825} & \textbf{0.848} & 3.103 & 1.148 \\
48 & 2.091 & \textbf{0.890} & 2.092 & 0.928 & 2.153 & 0.900 & 3.853 & 1.307 & 2.238 & 0.940 & 2.614 & 1.086 & 2.622 & 1.078 & \textbf{2.010} & 0.900 & 2.669 & 1.085 \\
60 & \textbf{1.999} & \textbf{0.897} & 2.018 & 0.921 & 2.029 & 0.910 & 4.035 & 1.344 & 2.027 & 0.928 & 2.804 & 1.146 & 2.857 & 1.157 & 2.178 & 0.963 & 2.770 & 1.125 \\
Avg. & 2.152 & \textbf{0.887} & \textbf{2.039} & 0.911 & 2.183 & 0.908 & 3.778 & 1.516 & 2.139 & 0.931 & 2.616 & 1.090 & 2.847 & 1.144 & 2.077 & 0.914 & 3.006 & 1.161 \\ \midrule
\end{tabular}%
}
\end{table*}

\subsection{Diversification Loss}
We show the similarity matrix in Fig. \ref{fig:div}, and obviously, the complementors learning with Diversification loss have become almost orthogonal with each other, suggesting that the diversity is maximization.

\begin{figure}[!t]
    \centering
    \includegraphics[width=0.9\linewidth]{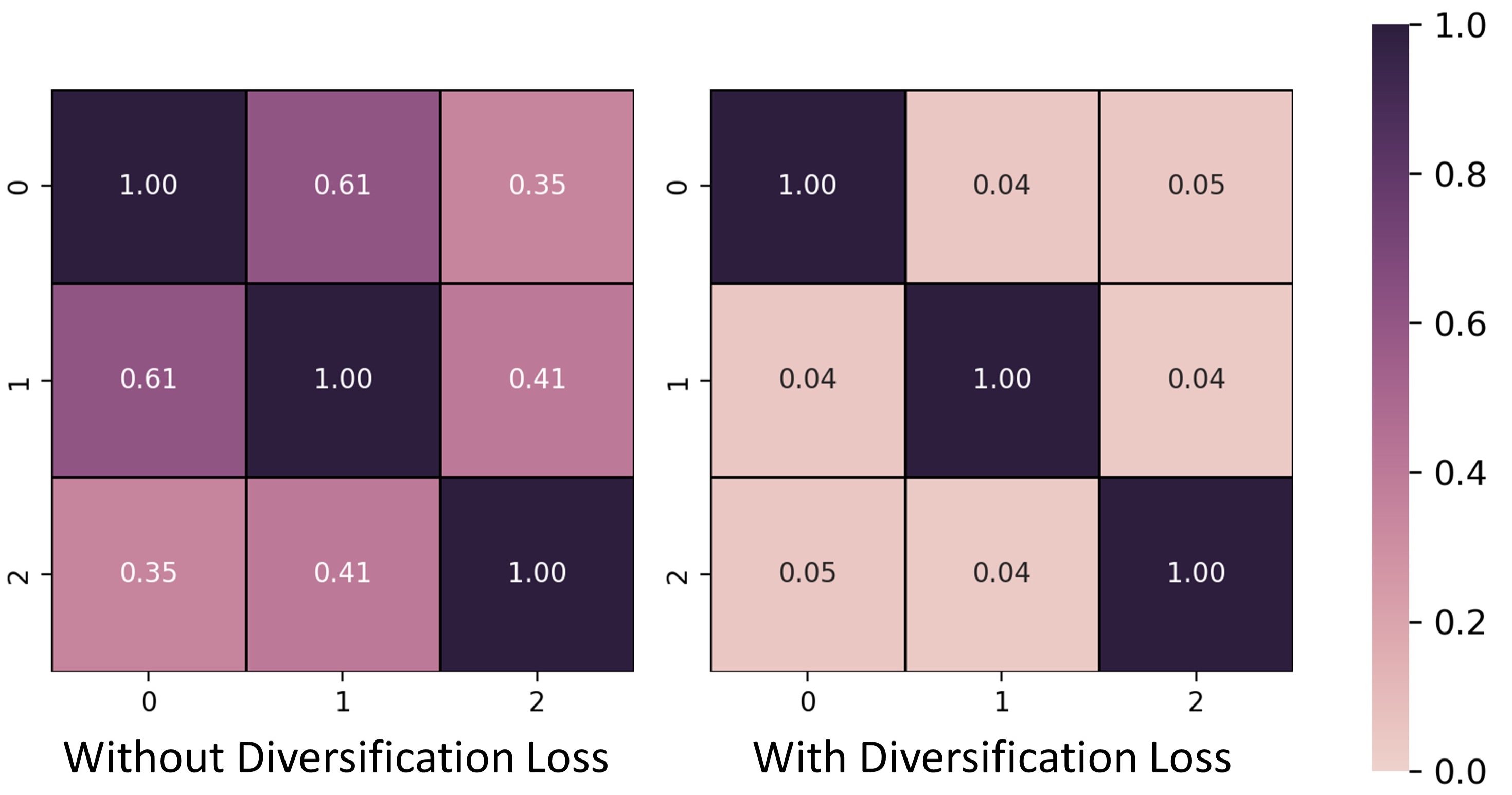}
    \caption{The similarity among Complementary Sequences learned without the proposed diversification loss and with the diversity loss. }
    \label{fig:div}
\end{figure}

\subsection{Case Study on iTransformer}
The full results of the case study on iTrnasformer are shown in Table \ref{tab:full_on_itrans}.
% We calculate the reduction as 1- (SOTA-withComp)/(SOTA-withoutComp). A value greater than 100\% implies that iTransformer with sequence complementors outperforms the previous SOTA. 
We also present their learned feature representation in Fig. \ref{fig:itrans_vis}, which illustrates the representation is enriched by adding the complementors.

\begin{figure}[!t]
    \centering
    \includegraphics[width=0.9\linewidth]{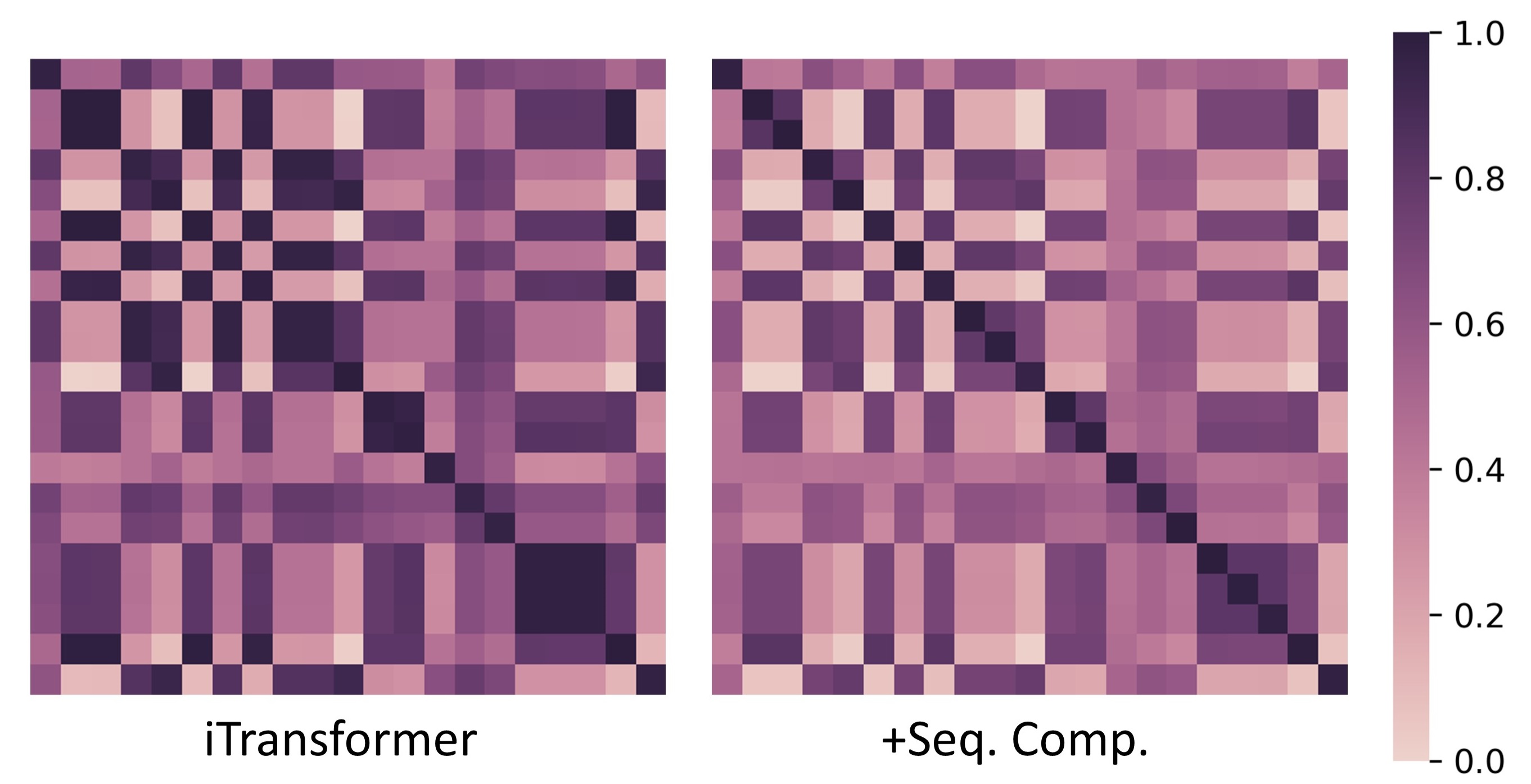}
    \caption{The visualization of learned feature representation by iTransformer. }
    \label{fig:itrans_vis}
\end{figure}

\

\begin{table*}[]
\label{tab:full_resuls}
\centering
\caption{The full results of the case study of iTransformer. 
% We put \textit{Previous SOTA*} as a reference here, which denotes the best performance for each setting across all models (exclude our proposed one).
}
\label{tab:full_on_itrans}
\resizebox{\textwidth}{!}{%
\begin{tabular}{cccccccccccccc}\toprule
 & Dataset & \multicolumn{2}{c}{ETTm1} & \multicolumn{2}{c}{ETTm2} & \multicolumn{2}{c}{ETTh1} & \multicolumn{2}{c}{ETTh2} & \multicolumn{2}{c}{Exchange} & \multicolumn{2}{c}{Weather} \\
 & Metric & MSE & MAE & MSE & MAE & MSE & MAE & MSE & MAE & MSE & MAE & MSE & MAE \\ \midrule
\multirow{5}{*}{\begin{tabular}[c]{@{}c@{}}iTransformer\\ \end{tabular}} & 96 & 0.334 & 0.368 & 0.180 & 0.264 & 0.386 & 0.405 & 0.297 & 0.349 & 0.086 & 0.206 & 0.174 & 0.214 \\
 & 192 & 0.377 & 0.391 & 0.250 & 0.309 & 0.441 & 0.436 & 0.380 & 0.400 & 0.177 & 0.299 & 0.221 & 0.254 \\
 & 336 & 0.426 & 0.420 & 0.311 & 0.348 & 0.487 & 0.458 & 0.428 & 0.432 & 0.331 & 0.417 & 0.278 & 0.296 \\
 & 720 & 0.491 & 0.459 & 0.412 & 0.407 & 0.503 & 0.491 & 0.427 & 0.445 & 0.847 & 0.691 & 0.358 & 0.347 \\
 & Avg & 0.407 & 0.410 & 0.288 & 0.332 & 0.454 & 0.447 & 0.383 & 0.407 & 0.360 & 0.403 & 0.258 & 0.278 \\ \midrule

 \multirow{5}{*}{\begin{tabular}[c]{@{}c@{}}+Seq. Comp. \end{tabular}} & 96 & \textbf{0.322} & \textbf{0.361} & \textbf{0.178} & \textbf{0.261} & \textbf{0.385} & \textbf{0.403} & \textbf{0.294} & \textbf{0.346} & \textbf{0.085} & \textbf{0.205} & \textbf{0.168} & \textbf{0.207} \\
 & 192 & \textbf{0.370} & \textbf{0.387} & \textbf{0.247} & \textbf{0.307} & \textbf{0.437} & \textbf{0.434} & \textbf{0.373} & \textbf{0.395} & \textbf{0.175} & \textbf{0.298} & \textbf{0.219} & \textbf{0.253} \\
 & 336 & \textbf{0.406} & \textbf{0.411} & \textbf{0.310} & \textbf{0.347} & \textbf{0.476} & \textbf{0.454} & \textbf{0.417} & \textbf{0.428} & \textbf{0.324} & \textbf{0.413} & \textbf{0.275} & \textbf{0.296} \\
 & 720 & \textbf{0.473} & \textbf{0.449} & \textbf{0.406} & \textbf{0.403} & \textbf{0.473} & \textbf{0.472} & \textbf{0.419} & \textbf{0.440} & \textbf{0.801} & \textbf{0.676} & \textbf{0.353} & \textbf{0.347} \\
 & Avg & \textbf{0.393} & \textbf{0.402} & \textbf{0.285} & \textbf{0.329} & \textbf{0.443} & \textbf{0.441} & \textbf{0.376} & \textbf{0.403} & \textbf{0.346} & \textbf{0.398} & \textbf{0.254} & \textbf{0.276} \\
%  \midrule
% \multirow{5}{*}{\begin{tabular}[c]{@{}c@{}}SOTA*\\ (exclude our proposed backbone)\end{tabular}} & 96 & 0.329 & 0.367 & 0.175 & 0.259 & 0.376 & 0.395 & 0.288 & 0.338 & 0.086 & 0.205 & 0.158 & 0.214 \\
%  & 192 & 0.367 & 0.385 & 0.241 & 0.302 & 0.420 & 0.424 & 0.374 & 0.390 & 0.176 & 0.299 & 0.206 & 0.254 \\
%  & 336 & 0.399 & 0.410 & 0.305 & 0.342 & 0.459 & 0.446 & 0.415 & 0.426 & 0.301 & 0.397 & 0.272 & 0.296 \\
%  & 720 & 0.454 & 0.439 & 0.402 & 0.398 & 0.481 & 0.470 & 0.420 & 0.440 & 0.839 & 0.691 & 0.345 & 0.347 \\
%  & Avg & 0.387 & 0.400 & 0.281 & 0.325 & 0.434 & 0.434 & 0.374 & 0.399 & 0.351 & 0.398 & 0.245 & 0.278 \\ 
 \bottomrule
\end{tabular}%
}
\end{table*}

\subsection{Additional Visualization}
We show the additional results with different forecasting lengths in Figs. \ref{fig:ETTh2_96_case0} to \ref{fig:ETTm2_720_case8} \underline{at the \textbf{bottom}}.

\begin{figure*}[!t]
    \centering
    \includegraphics[width=0.9\linewidth]{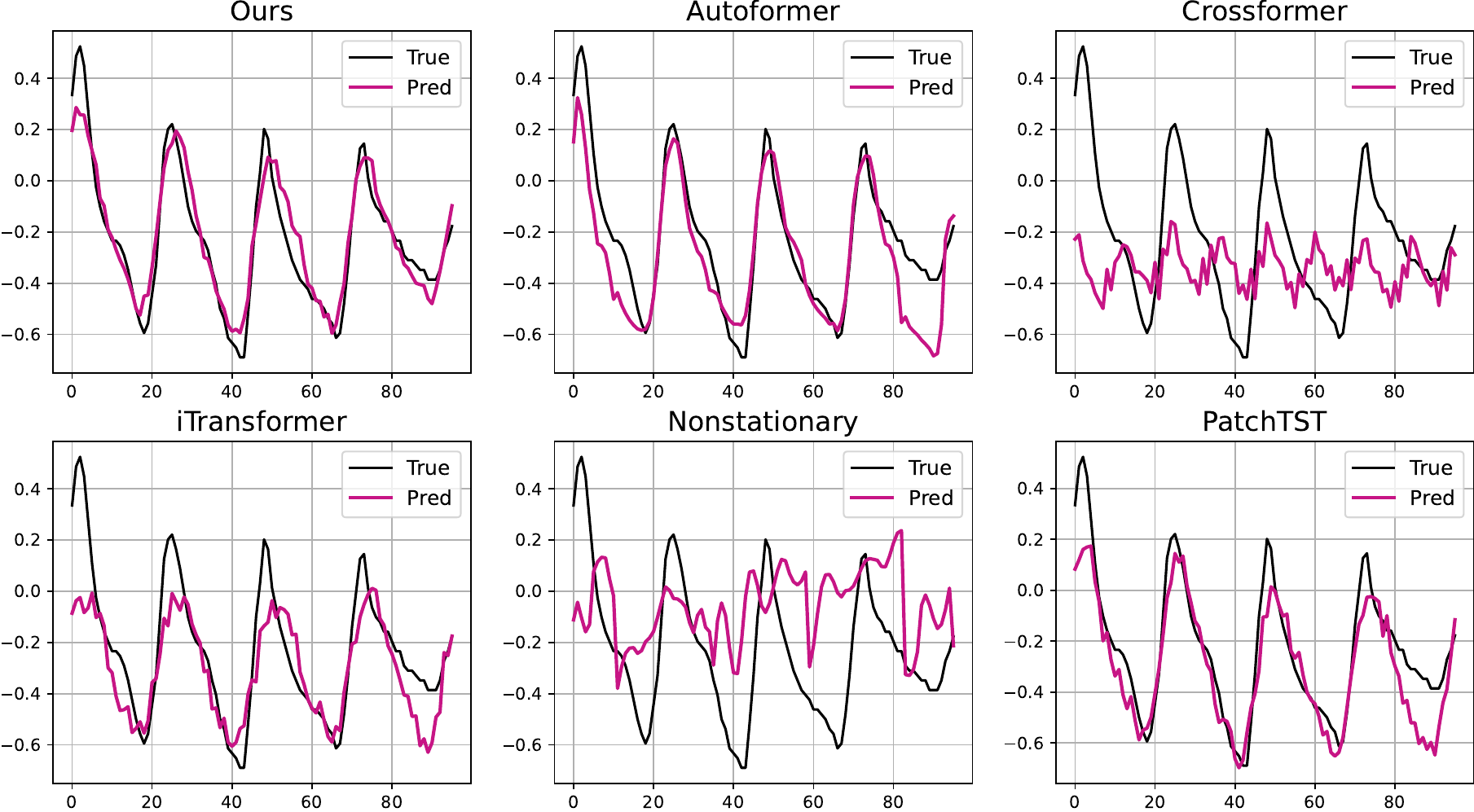}
    \caption{The visualization of forecasting results on the ETTh2 dataset with a forecasting length of 96. }
    \label{fig:ETTh2_96_case0}
\end{figure*}

\begin{figure*}[!t]
    \centering
    \includegraphics[width=0.9\linewidth]{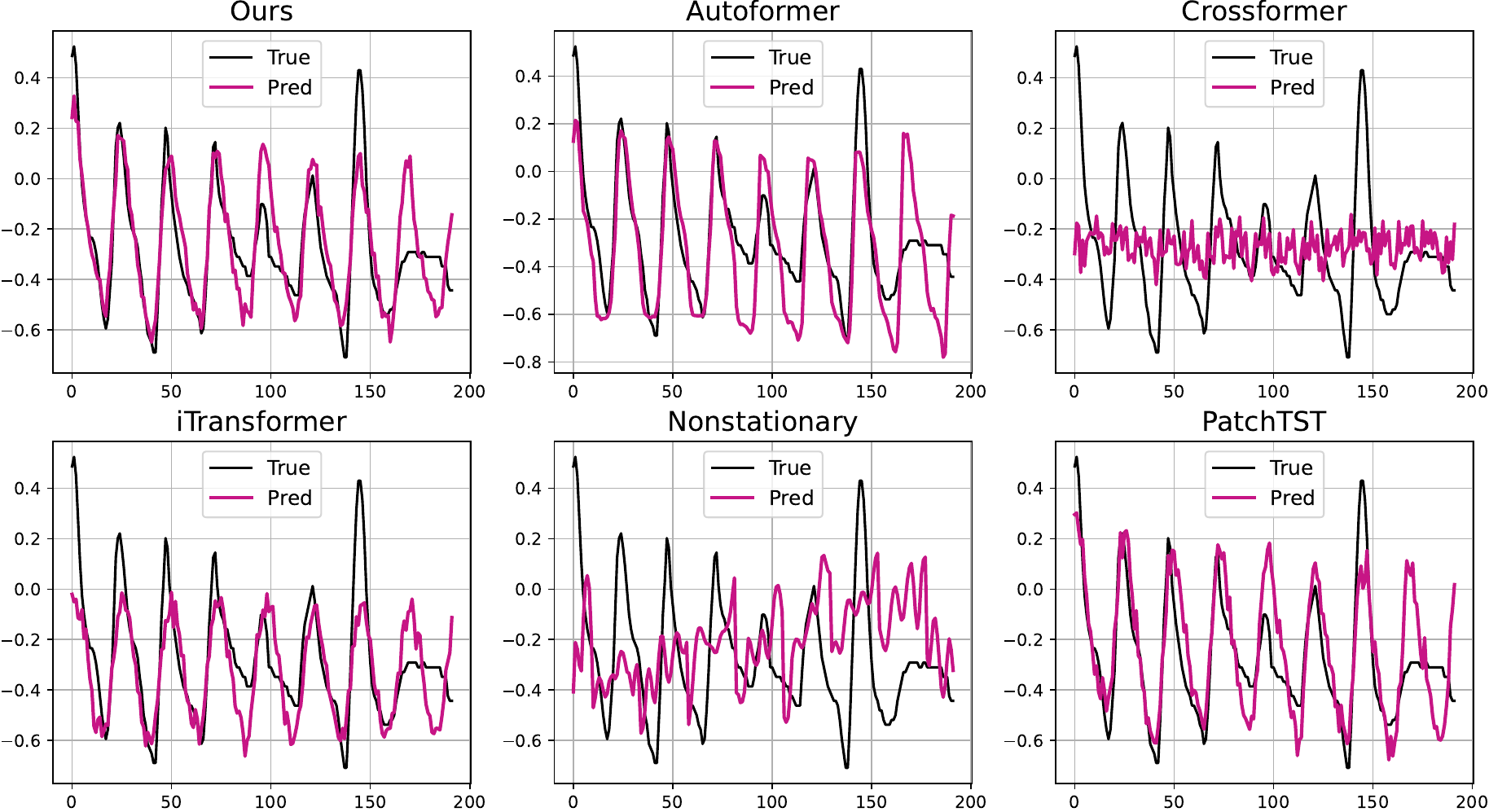}
    \caption{The visualization of forecasting results on the ETTh2 dataset with a forecasting length of 192. }
    \label{fig:ETTh2_192_case3}
\end{figure*}

\begin{figure*}[!t]
    \centering
    \includegraphics[width=0.9\linewidth]{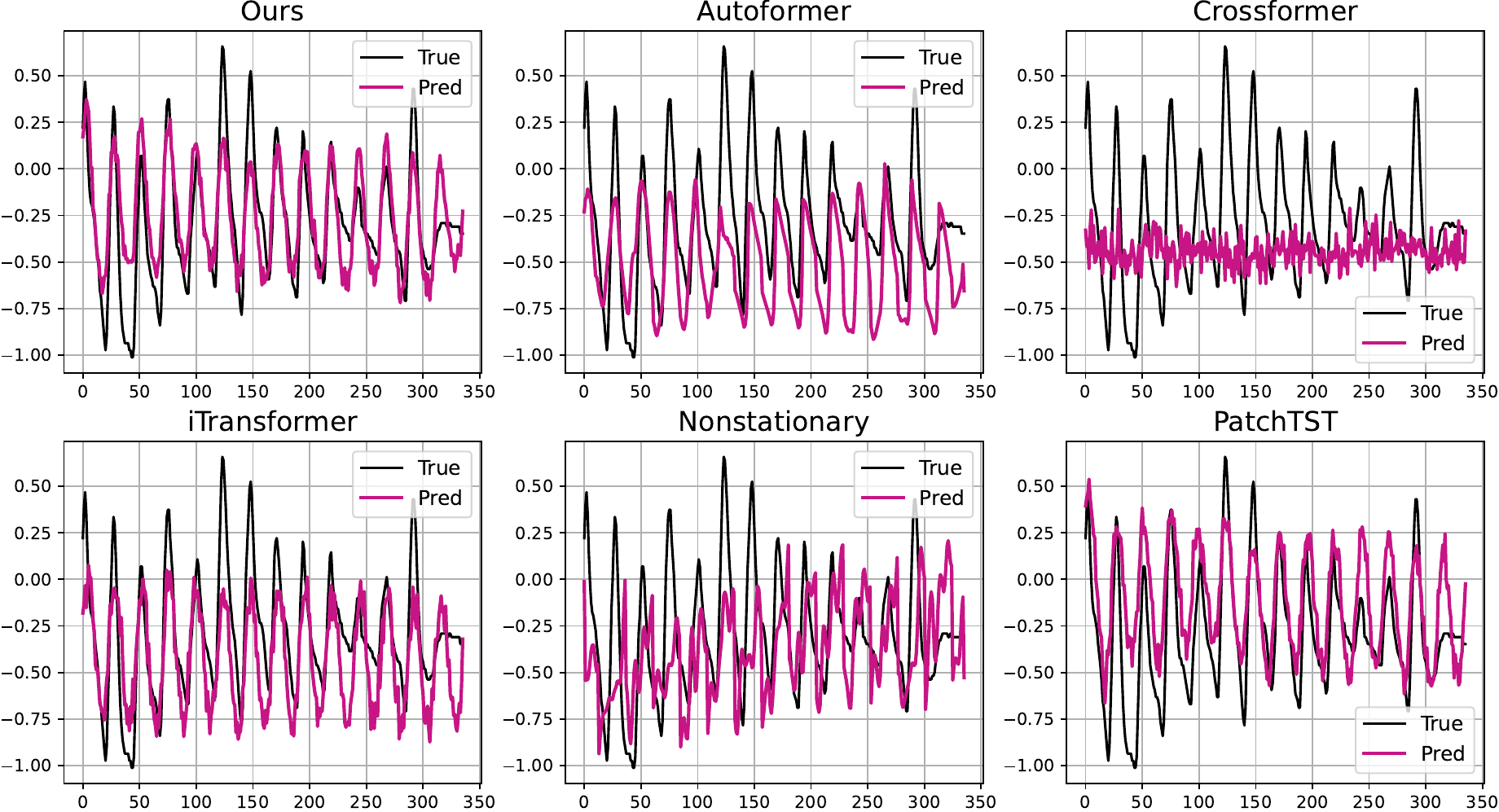}
    \caption{The visualization of forecasting results on the ETTh2 dataset with a forecasting length of 336. }
    \label{fig:ETTh2_336_case9}
\end{figure*}

\begin{figure*}[!t]
    \centering
    \includegraphics[width=0.9\linewidth]{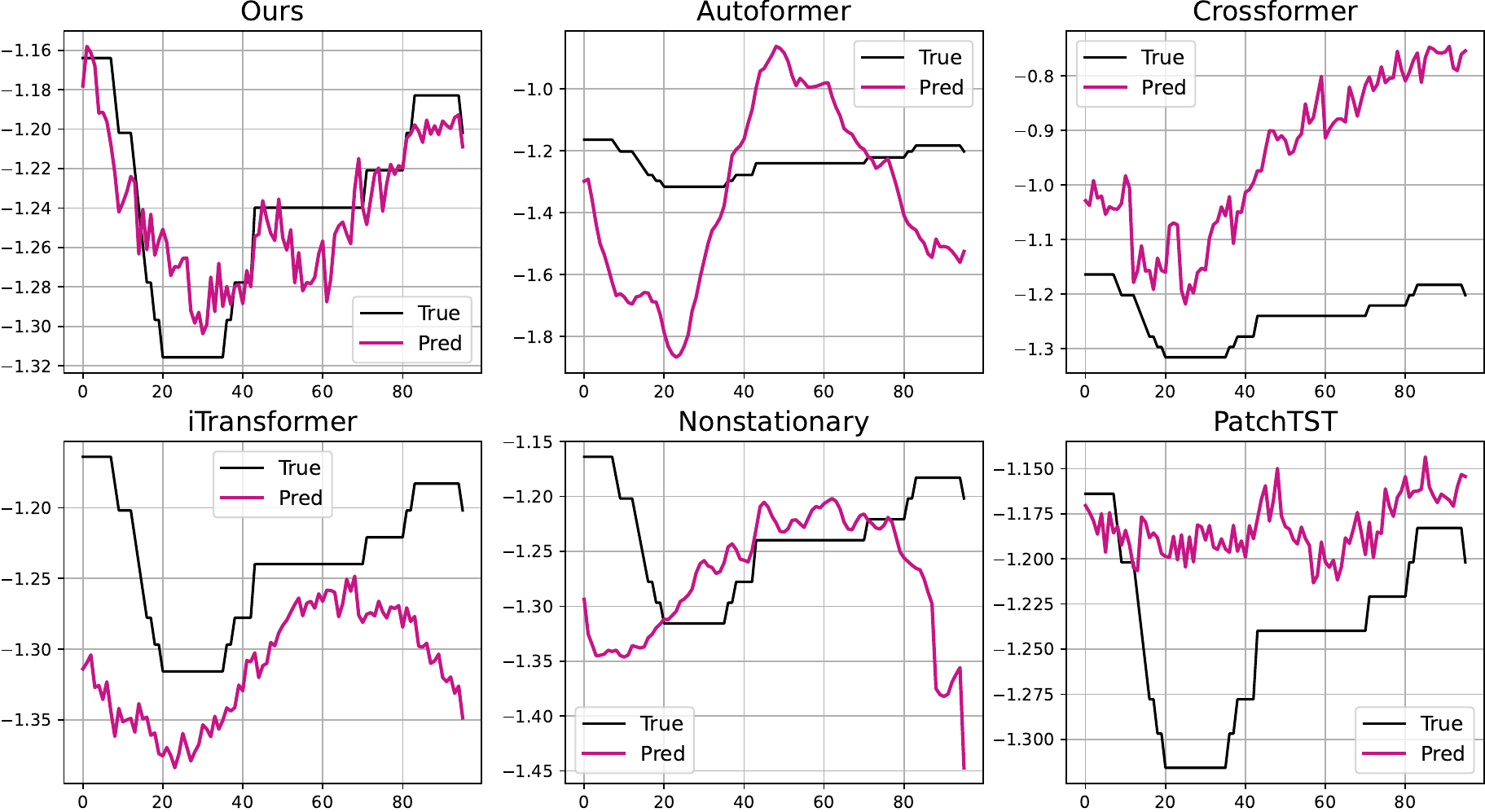}
    \caption{The visualization of forecasting results on the ETTm2 dataset with a forecasting length of 96. }
    \label{fig:ETTm2_96_case0}
\end{figure*}

\begin{figure*}[!t]
    \centering
    \includegraphics[width=0.9\linewidth]{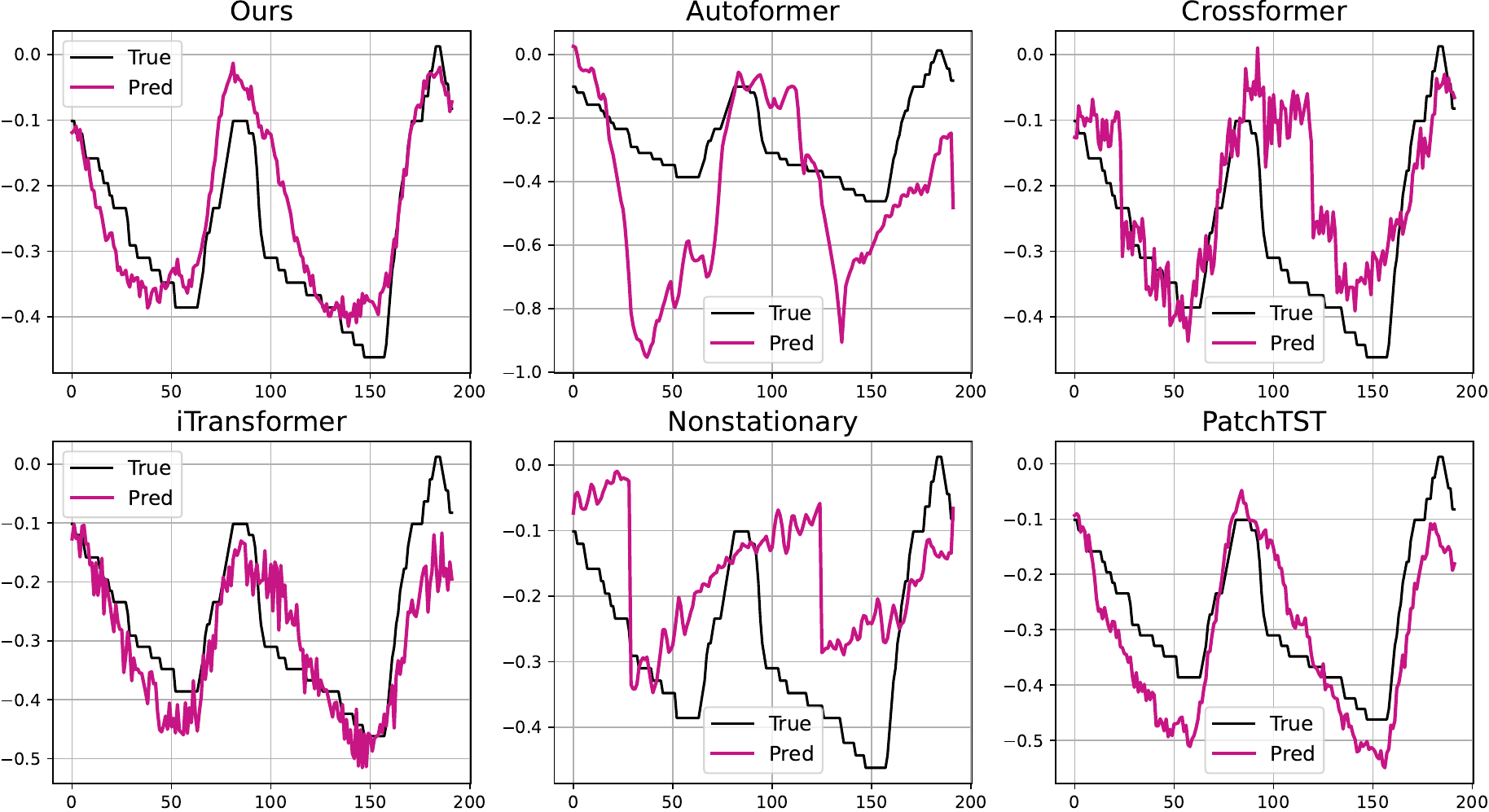}
    \caption{The visualization of forecasting results on the ETTm2 dataset with a forecasting length of 192. }
    \label{fig:ETTm2_192_case5}
\end{figure*}

\begin{figure*}[!t]
    \centering
    \includegraphics[width=0.9\linewidth]{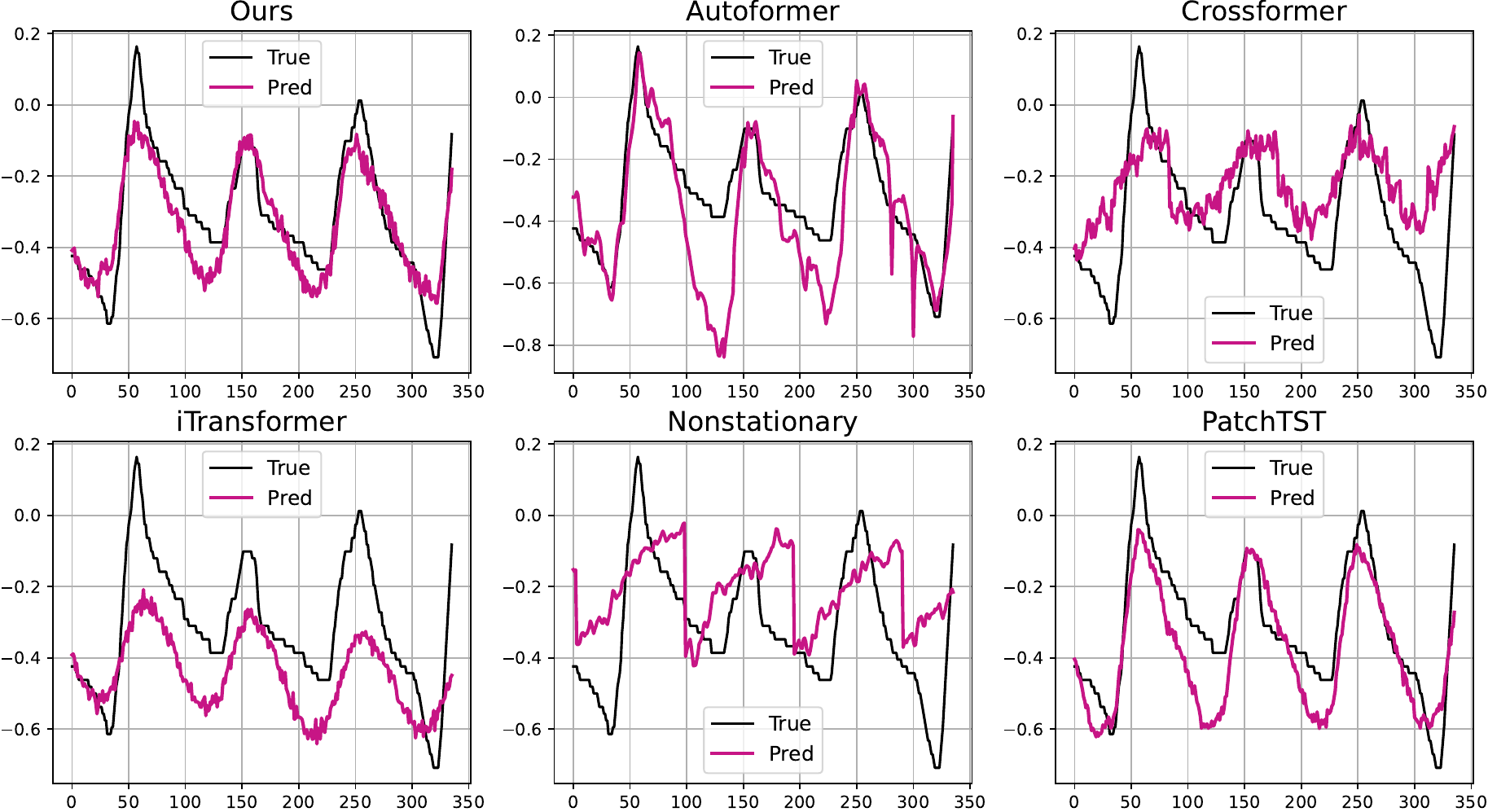}
    \caption{The visualization of forecasting results on the ETTm2 dataset with a forecasting length of 336. }
    \label{fig:ETTm2_336_case3}
\end{figure*}

\begin{figure*}[!t]
    \centering
    \includegraphics[width=0.9\linewidth]{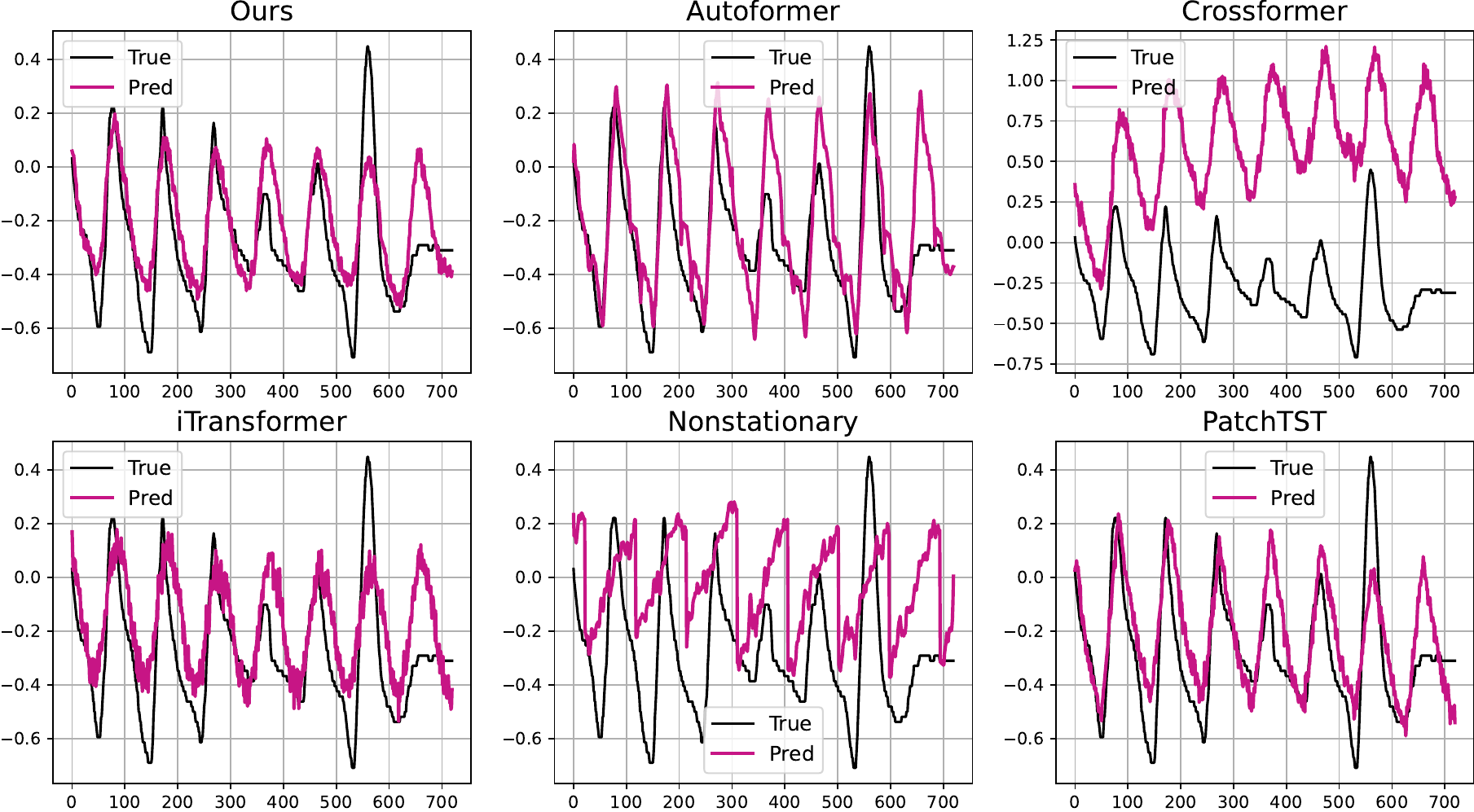}
    \caption{The visualization of forecasting results on the ETTm2 dataset with a forecasting length of 720. }
    \label{fig:ETTm2_720_case8}
\end{figure*}

\section{Discussion}

\subsection{Transformer v.s. Linear Model in TSF}
While the effectiveness of transformers for TSF has empirically been challenged by simple Linear models in recent works~\cite{li2023revisiting,zeng2023transformers}, we argue that they still have great potential for TSF due to their inherent ability to capture long-range dependencies. This can be articulated from several perspectives. 
First, transformers generally require a substantial amount of data due to the quadratic complexity of self-attention, indicating that they need more training data to fully leverage their advantages. Our empirical studies also supported this claim that transformers are struck to learn rich and generalizable feature representations with limited data. In contrast, the linear TSF models [e.g., models proposed by~\cite{li2023revisiting,zeng2023transformers}] may also demonstrate good performance under the data-scarce conditions merely because of their simplicity. However, it is overstated that Linear TSF is better than transformers. This has also been supported by the empirical studies in~\cite{zeng2023transformers}, which show that the transformer can obtain a performance improvement over the linear model even with a slight increase in dataset size. 
Second, the enhancement in the architectural design of transformers may not improve the learning ability of transformers in data-scarce conditions from its root, though we admit that the improvement of model architectures may also improve their performance and learning ability. As suggested by our studies, simply injecting more learnable sequences can help transformers learn richer and more generalizable feature representations. 

% While enhancing Transformers with additional modules is mainstream, recent studies increasingly question transformer effectiveness, suggesting simpler models (e.g. Linear mapping) can also excel in MTSF~\cite{li2023revisiting,zeng2023transformers}. However, after the Dlinear \cite{zeng2023transformers}, PatchTST is developed with simple patchify and channel-independent strategies, demonstrating a substantial improvement over all previous works (including the linear method).
% Additionally, it's well-known that Transformers are data-hungry, making this claim unfair for MTSF benchmarks with limited data. Some people \cite{wang2024card} from the community also believe the claim made by \cite{zeng2023transformers} that the transformer model cannot benefit from more data is overstatement. They have illustrated that even with a slight increase in dataset size, the transformer can obtain a performance improvement over the linear model. Therefore, to accurately assess how well the model scales with data, it is necessary to conduct investigations using large datasets.
% We can also give a very extreme example, Time-LLM \cite{jin2024timellm}, which is based on large language models, and that model is apparently built on Transformer. Its empirical results suggest that training transformers on large data demonstrates unprecedented achievement on main time series tasks.
% Hence, in summary, we still believe Transformers have unlimited potential in MTSF, and its power has not been fully explored. 

\subsection{Limitation and Future Work}
We realized that the metrics we used to estimate the diversity are promising; however, it is still challenging to apply them to compare the feature with different dimensions and different samples, especially when the number of samples is limited, and the dimension is high. Therefore, we plan to investigate this matter further to make a fairer comparison. 
We also realized that we are still adopting a channel-independent strategy as a proof-of-concept to demonstrate the effectiveness of our proposed \texttt{Complementor}. We believe the performance can be further improved by including some components, such as encoding cross-variate information, deformable patchifying strategy, and better layer normalization. As these are out of our scope in our paper, we will investigate them in the future.

\end{document}